\documentclass{article}

\usepackage[preprint]{neurips_2026}

\usepackage[utf8]{inputenc} 
\usepackage[T1]{fontenc}    
\usepackage{hyperref}       
\usepackage{url}            
\usepackage{booktabs}       
\usepackage{amsfonts}       
\usepackage{nicefrac}       
\usepackage{microtype}      
\usepackage[table]{xcolor}  

\usepackage{graphicx}
\usepackage{amsmath}
\definecolor{rowblue}{RGB}{235, 240, 255}

\usepackage{amsthm}
\usepackage{amssymb}      
\usepackage{longtable}    
\usepackage{cleveref}     
\usepackage{multirow}
\usepackage{caption}
\usepackage{wrapfig}

\newtheorem{theorem}{Theorem}
\newtheorem{lemma}[theorem]{Lemma}
\newtheorem{corollary}[theorem]{Corollary}
\newtheorem{assumption}{Assumption}
\newtheorem{definition}{Definition}
\newtheorem{remark}{Remark}

\title{Matryoshka Concept Bottleneck Models}

\author{%
\begin{tabular}{c}
\textbf{Ziye Chen}$^{*,1,2}$\quad
\textbf{Hongbin Lin}$^{*,2}$\quad
\textbf{Jie Li}$^{1}$\quad
\textbf{Lijie Hu}$^{1}$ \\
\normalfont $^{1}$Mohamed bin Zayed University of Artificial Intelligence \\
\normalfont $^{2}$The Hong Kong University of Science and Technology (Guangzhou) \\
\normalfont $^{3}$The Hong Kong University of Science and Technology
\end{tabular}
}

\makeatletter
\renewcommand{\@noticestring}{%
  Preprint. $^*$Equal contribution. Correspondence to: Lijie Hu \texttt{<lijie.hu@mbzuai.ac.ae>}%
}
\makeatother

\begin{document}

\maketitle

\begin{abstract}
Concept Bottleneck Models (CBMs) have emerged as a prominent paradigm for interpretable deep learning, learning by grounding predictions in human-understandable concepts. However, their practical deployment is hindered by the high cost of test-time intervention, as correcting model errors typically requires human experts to manually inspect and verify a large set of predicted concepts. 
Existing approaches suffer from a fundamental structural limitation: they either adopt a single static concept set, forcing experts to exhaustively annotate concepts and incurring prohibitive intervention costs, or train multiple models tailored to different concept budgets, resulting in substantial computational and maintenance overhead.
To address this challenge, we propose the \textbf{Matryoshka Concept Bottleneck Model (MCBM)}, a unified architecture that enables adaptive concept utilization within a single model. Inspired by Matryoshka Representation Learning, MCBM organizes concepts into a nested hierarchy based on maximum relevance and minimum redundancy, allowing inference at multiple levels of conceptual granularity without retraining. 
Theoretically, we show that MCBM reduces the expected intervention costs from linear to logarithmic order, $O(\log K)$, while guaranteeing monotonic performance improvement. 
Empirically, extensive experiments demonstrate that MCBM matches the performance of independently trained models while enabling dynamic and efficient expert interaction.
\end{abstract}

\section{Introduction}
\label{sec:intro}

As deep learning models are increasingly deployed in high-stakes domains such as healthcare and finance, transparency and interpretability have become central concerns~\cite{lin2026controllable, hueditable}. Concept Bottleneck Models (CBMs)~\cite{koh2020concept} address this challenge by decomposing predictions into two interpretable stages: mapping inputs to human-understandable concepts and making predictions solely based on these concepts.
Beyond providing explanations, CBMs enable human experts to intervene by inspecting and correcting erroneous concept predictions, motivating growing interest in effective and interactive CBM design~\cite{ECBM, hu2024towards,xu2025concept}.

However, the practical utility of CBMs is hindered by the high cost of test-time intervention. Existing CBM approaches treat concepts as a flat and unordered set, requiring experts to manually inspect and correct many concepts to fix a single prediction~\cite{Chauhan2022}, rendering large-scale deployment impractical. While reducing the number of concepts can lower this burden, it often leads to a substantial drop in predictive performance~\cite{havasi2022addressing, hu2025stable, hu2025semi, cheng2025compke}. As a result, practitioners face a rigid design choice: either deploy a single high-dimensional CBM that is costly to intervene on, or train and maintain multiple models with different concept budgets, incurring significant computational overhead and lacking flexibility for real-world deployment~\cite{zhang2025locate, lai2024faithful}.

This rigidity arises from how concepts are structured and utilized in existing CBMs. In fine-grained tasks, many concepts are redundant, yet standard CBM approaches treat them as independent and equally important, leading to inefficient intervention. A natural solution is to exploit the information-theoretic structure of concept sets. Organizing concepts into a nested hierarchy based on \textbf{Maximum Relevance and Minimum Redundancy (mRMR)} concentrates discriminative information in early concept dimensions. This organization aligns with Matryoshka Representation Learning, where representations are ordered by decreasing importance within a single model, offering a principled route to efficient and adaptive intervention.

To operationalize this insight, we introduce the \textbf{Matryoshka Concept Bottleneck Model (MCBM)}. As illustrated in Figure~\ref{fig:mcbm_architecture}, MCBM is a unified architecture designed for flexible and cost-efficient deployment. Unlike existing methods that require training separate models for different concept budgets, our framework is trained once and supports adaptive inference across a wide range of intervention constraints. Specifically, we employ the mRMR criterion to construct a hierarchical ordering of concepts, ensuring that the most informative and least redundant concepts are prioritized. The model is then optimized using a nested objective function that enforces predictive accuracy at multiple levels of conceptual granularity simultaneously. Consequently, MCBM functions as a plug-and-play module, allowing practitioners to dynamically trade off annotation cost and predictive performance without retraining.


Our main contributions are summarized as follows:
\begin{itemize}
    \item \textbf{Methodology:} We introduce MCBM, a unified architecture that integrates mRMR-based concept ordering with Matryoshka Representation Learning, enabling elastic deployment and efficient intervention from a single training run.
    \item \textbf{Theoretical Analysis:} We show that the proposed nested, importance-based structure reduces the expected intervention cost from linear $O(K)$ to logarithmic $O(\log K)$ under standard information decay assumptions.
    \item \textbf{Experiments:} We empirically validate MCBM on the CUB, LAD, and CelebA datasets, demonstrating performance comparable to independently trained baselines with substantially reduced intervention effort.
\end{itemize}

\begin{figure*}[h]
    \centering
    \includegraphics[width=1\linewidth]{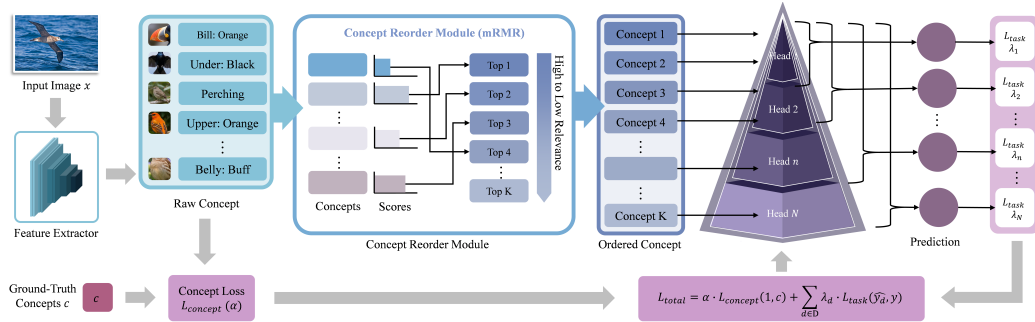}
    \caption{\textbf{Matryoshka Concept Bottleneck Models Architecture.} The input image is encoded into raw logits, which are then permuted based on the pre-computed mRMR ranking. This yields an ordered concept vector where information density is concentrated at the beginning. Multiple parallel heads (Matryoshka Heads) then perform classification using nested prefixes of this ordered vector.}
    \label{fig:mcbm_architecture}
\vspace{-7pt}
\end{figure*}

\section{Related Work}
\label{sec:related_work}

\paragraph{Concept Bottleneck Models and the Rigid Cost Paradigm.}
Concept Bottleneck Models (CBM)~\cite{koh2020concept} have established themselves as a cornerstone of interpretable deep learning, particularly in high-stakes domains like healthcare~\cite{ahmad2018interpretable,yu2018artificial} and dermatology~\cite{daneshjou2021dermatologist}, where distinguishing specific attributes is critical for trust. The core appeal of CBM lies in their ability to support test-time intervention~\cite{shin2023closer}, allowing human experts to rectify mispredictions.
Recent work improves CBM along orthogonal axes such as robustness, label-efficiency, faithfulness, and concept-set completeness~\cite{havasi2022addressing, oikarinen2023labelfree, yuksekgonul2022post, zarlenga2023learning, lai2024faithful, bhan2025complete}. 
However, despite these advances in model quality, existing approaches fundamentally neglect the \emph{inference efficiency} of the bottleneck itself. Whether supervised or label-free~\cite{oikarinen2023labelfree, yang2023language}, these models structure concepts as a \emph{flat, unordered bottleneck}, imposing a rigid $O(K)$ verification cost on the user regardless of the input's difficulty. Our Matryoshka CBM (MCBM) addresses this overlooked dimension by introducing a nested hierarchy, enabling efficient, logarithmic-cost interventions without sacrificing the robustness or faithfulness gains of prior arts.

\paragraph{Sparse, Hybrid, and Adaptive CBM.}
To mitigate the computational and cognitive overhead of CBM, a second line of work introduces sparsity or hybrid mechanisms. \emph{Hybrid CBM} combine concept bottlenecks with standard black-box branches to recover performance~\cite{antognini2021concept, havasi2022addressing}, but often at the cost of full interpretability. \emph{Sparse CBM} utilize regularization techniques (e.g., $L_1$ penalties, Gumbel-Softmax) to dynamically select a small subset of active concepts per input~\cite{semenov2024sparse, wong2021leveraging}. Similarly, \emph{Adaptive CBM} attempt to adjust the bottleneck capacity to handle distribution shifts~\cite{chowdhury2024adacbm}. 
However, these methods typically result in \emph{unstructured sparsity} where the model activates a disparate set of concepts for each individual sample. This lack of a globally consistent ordering prevents systematic resource truncation (e.g., for deployment on edge devices with fixed memory budgets). Furthermore, they lack a theoretical guarantee for intervention efficiency. In contrast, MCBM enforces a \emph{globally consistent, importance-based ordering} via mRMR, enabling predictable plug-and-play deployment and theoretically bounded intervention costs.

\paragraph{Causal and Stochastic CBM.}
Causal and stochastic CBMs model dependencies among concepts to improve intervention propagation and uncertainty-aware reasoning~\cite{vandenhirtz2024stochastic,defelice2026causal}. This direction is complementary to ours. Such models answer how a correction should propagate through a concept graph, while MCBM answers which concepts should be queried first under a strict human or hardware budget~\cite{rafferty2026radiologist,fokkema2026sample,kalampalikis2025towards}. Even when a causal graph identifies task-relevant parents, it does not by itself solve the redundancy-aware subset-ordering problem; highly correlated parents can waste a small intervention budget. Moreover, graph-based propagation usually requires the full concept graph and full vocabulary to remain active at inference time. MCBM instead bakes a static mRMR ordering into the architecture, allowing unused suffix concepts and their heads to be physically truncated for predictable low-resource deployment.

\paragraph{Matryoshka Representation Learning and Feature Selection.}
Beyond CBM, \textit{Matryoshka Representation Learning (MRL)}~\cite{kusupati2022matryoshka} enforces nested embeddings for multi-granular inference, while Minimum Redundancy Maximum Relevance (mRMR)~\cite{peng2005feature, ding2005minimum,bussmann2025learning,lai2026matryoshka} orders features to maximize information density. However, MRL operates on opaque latent embeddings unsuitable for human intervention, and mRMR is traditionally limited to tabular data. We unify MRL's structural efficiency with mRMR's information-theoretic rigor inside a semantic CBM, yielding a model that is simultaneously interpretable, intervention-efficient, and structurally elastic.

\section{Matryoshka Concept Bottleneck Models}
\label{sec:method}

In interpretable deep learning, a fundamental tension exists between \textit{model transparency} and \textit{deployment efficiency}. Since runtime concept annotation is costly, standard Concept Bottleneck Models (CBM) face a rigid dilemma: utilizing all concepts incurs prohibitive intervention costs, while discarding them degrades performance. Furthermore, adapting to different resource budgets typically requires training multiple specialized models, creating significant computational overhead.

We introduce the Matryoshka Concept Bottleneck Models (MCBM), which reimagine the concept layer as a \textit{nested information hierarchy} ordered by global importance. A single unified model then supports inference at varying granularities (e.g., top-8, top-16, or all concepts) without retraining, decoupling training from deployment constraints.

\subsection{Architecture Overview}

The MCBM data flow during end-to-end training follows a three-step pipeline: (1) \textbf{Concept Ranking}, which discovers a global topological ordering of concepts; (2) \textbf{Matryoshka Encoding}, which projects visual features into this ordered semantic space; and (3) \textbf{Nested Prediction}, where parallel heads jointly optimize accuracy across varying concept budgets.

\subsubsection{Concept Importance Ranking via mRMR}
\label{subsubsec:mrmr_ranking}

\textbf{Intuition.} 
Before training, we fix a hierarchical ordering of all concepts. Ranking concepts by individual discriminative power is insufficient: in fine-grained datasets, top-scoring concepts are often highly correlated (e.g., ``has wings'' and ``flight capable''), so a naively ranked prefix carries overlapping information and wastes the limited capacity of low-dimensional bottlenecks. We therefore seek a permutation that maximizes the \textit{information density} of every prefix.

\textbf{Formulation.} 
We employ the Minimum Redundancy Maximum Relevance (mRMR) criterion, a greedy procedure that iteratively picks the concept maximizing relevance to the target while minimizing redundancy with the already selected set. Let $\Omega$ be all concepts and $\mathcal{S}_{t-1}$ the $t-1$ already chosen; the $t$-th concept is selected from $\Omega \setminus \mathcal{S}_{t-1}$ as:
\begin{equation}
    c_t = \operatorname*{argmax}_{c_i \in \Omega \setminus \mathcal{S}_{t-1}} \left( \underbrace{I(c_i; y)}_{\text{Relevance}} - \frac{1}{|\mathcal{S}_{t-1}|} \sum_{c_j \in \mathcal{S}_{t-1}} \underbrace{I(c_i; c_j)}_{\text{Redundancy}} \right)
\end{equation}
where $I(\cdot; \cdot)$ denotes the Mutual Information (MI). For discrete variables (binary concepts and categorical labels), MI is defined as:
\begin{equation}
    I(X; Y) = \sum_{x \in X} \sum_{y \in Y} P(x, y) \log \frac{P(x, y)}{P(x)P(y)}
\end{equation}
where $P(x, y)$ is the joint probability distribution and $P(x), P(y)$ are the marginals.

\textbf{Interpretation.} 
The mRMR score explicitly balances two competing objectives to ensure information efficiency:
\begin{itemize}
    \item \textbf{Max-Relevance ($I(c_i; y)$):} Prioritizes concepts that maximize the mutual information with the target label $y$, ensuring high predictive power.
    \item \textbf{Min-Redundancy ($I(c_i; c_j)$):} Penalizes candidate concepts that share high mutual information with the already selected set $\mathcal{S}_{t-1}$, preventing information duplication.
\end{itemize}

Subtracting the redundancy term acts as a regularizer that favors \textit{complementary} concepts over discriminative but repetitive ones, and using mutual information instead of linear correlation captures non-linear semantic dependencies. The selection runs once \textit{before} network training at negligible $O(K^2)$ cost, producing a fixed topological blueprint that forces the model to populate the most critical semantic channels first.

\textbf{Intuitive Example.} For bird classification, ``has red breast'' and ``has red chest'' are both relevant for a Robin but nearly redundant. After selecting ``has red breast'', mRMR drops the marginal gain of ``has red chest'' and instead picks, e.g., ``beak shape'', so the first two dimensions cover diverse attributes (color \textit{and} shape) and maximize accuracy of the low-dimensional Matryoshka heads.

\subsubsection{Matryoshka Architecture and Joint Optimization}

\textbf{Intuition.}
Building upon the mRMR-optimized hierarchy, we design the neural architecture to enforce the ``Matryoshka property'': the representation at dimension $d$ should be a sufficient statistic for the task, even if dimensions $d+1 \dots K$ are truncated. This implies that the model cannot distribute information arbitrarily; it must compress the most vital semantic signals into the earliest slots defined by our ranking.

\textbf{Forward Pass and Heads.}
An input image $\mathbf{x}$ is mapped by a feature extractor $g_{\theta}(\cdot)$ to a raw concept-logit vector $\mathbf{l} \in \mathbb{R}^{K}$, then permuted by a deterministic module $\Pi$ following the mRMR indices into the ordered vector $\mathbf{c}_{\text{ord}} = \Pi(\mathbf{l})$. A set of parallel lightweight heads $\{f_{\phi_d}\}_{d \in \mathcal{D}}$, with nesting granularities $\mathcal{D} = \{d_1, \dots, d_n\}$ (e.g., $\{8, 16, 32, \dots\}$), each act as a linear classifier on the corresponding prefix:
\begin{equation}
    \hat{y}_d = f_{\phi_d}(\mathbf{c}_{\text{ord}}[1:d]) = \mathbf{W}_d \cdot \mathbf{c}_{\text{ord}}[1:d] + \mathbf{b}_d
\end{equation}
Because the input is ordered by mRMR, the head $f_{\phi_d}$ effectively utilizes the maximum possible information encodable in $d$ bits.

\textbf{Joint Optimization Objective.}
A joint loss trains the backbone and all heads simultaneously, combining concept alignment with multi-granularity task performance:
\begin{equation}
    \mathcal{L}_{\text{total}} = \alpha \cdot \mathcal{L}_{\text{concept}}(\mathbf{l}, \mathbf{c}_{GT}) + \sum_{d \in \mathcal{D}} \lambda_d \cdot \mathcal{L}_{\text{task}}(\hat{y}_d, y_{GT})
\end{equation}

In the standard MCBM, the encoder, permutation module, and all prediction heads are optimized in a single end-to-end run under Eq.~(3). Each mini-batch supplies concept supervision to preserve human semantics while every head imposes task supervision at its own prefix length, so the model learns one ordered bottleneck whose early coordinates remain useful for every downstream head, rather than stitching together isolated CBMs.

\textbf{Interpretation of the Loss Dynamics.}
The first component, \textbf{Concept Grounding ($\mathcal{L}_{\text{concept}}$)}, is a binary cross-entropy loss against the ground-truth concepts $\mathbf{c}_{GT}$. It anchors the bottleneck to human-defined semantics; without it, the encoder may treat the bottleneck as arbitrary latent variables and lose interpretability.

The second component, \textbf{Nested Task Supervision ($\sum \mathcal{L}_{\text{task}}$)}, is the engine of Matryoshka learning. Gradients from every head $f_{\phi_d}$ backpropagate to the shared encoder, so the earliest dimensions of $\mathbf{c}_{\text{ord}}$ accumulate updates from \textit{all} heads in $\mathcal{D}$ while later dimensions receive gradients only from the largest head. This differential pressure forms a ``gravity well'' that pushes the most robust, generalizable features into the early dimensions, since they must serve every head simultaneously.

\subsection{Efficient MCBM: Trading Performance for Memory}
\label{subsec:efficient_mcbm}

\textbf{Intuition.}
Standard MCBM uses distinct weights $\mathbf{W}_d$ per nesting level, allowing each head to specialize but incurring a memory cost linear in $|\mathcal{D}|$ and limiting inference to predefined checkpoints. Under tight memory budgets or ``any-time'' inference (stopping at \textit{any} integer $k$), this becomes wasteful, motivating the \textbf{Efficient MCBM} variant.

\textbf{Mechanism.}
Instead of instantiating independent physical heads $\{f_{\phi_d}\}$, the Efficient MCBM utilizes a \textit{single, shared weight matrix} $\mathbf{W}_{\text{shared}} \in \mathbb{R}^{C \times K}$ (where $C$ is the number of classes and $K$ is the total concept count). To emulate the Matryoshka behavior, we apply a dynamic masking operation during training.
For a given nesting level $d$, we construct a binary mask $\mathbf{M}_d \in \{0, 1\}^K$ where the first $d$ entries are 1 and the rest are 0. The prediction at level $d$ is computed as:
\begin{equation}
    \hat{y}_d = (\mathbf{W}_{\text{shared}} \odot \mathbf{M}_d) \cdot \mathbf{c}_{\text{ord}}
\end{equation}
where $\odot$ denotes row-wise broadcasting of the mask. During training, our default implementation iterates through all levels in $\mathcal{D}$ and sums their losses for each batch, giving the lowest-dimensional prefixes dense gradient signals. A lighter random-level variant can reduce training cost, but we treat it as an efficiency ablation rather than the default because extreme compression levels are most sensitive to sparse supervision.

\textbf{Implications and Trade-offs.}
Weight sharing forces a single matrix to be jointly optimal for early predictions and later refinements, collapsing parameter overhead to $O(C \cdot K)$ and enabling \textit{continuous elasticity} where inference can stop at any integer $k$. The cost is a small accuracy gap (Section~\ref{sec:experiments}), since shared weights cannot specialize per scale, giving practitioners a clean choice between peak accuracy (Standard) and minimum memory footprint (Efficient).

\section{Theoretical Analysis}
\label{sec:theory}

Our empirical results demonstrate that MCBM achieves high accuracy with low intervention costs. In this section, we provide a theoretical foundation for these observations from two perspectives: (1) \textbf{Efficiency}, determining the asymptotic scaling of intervention cost, and (2) \textbf{Correctness}, establishing an upper bound on classification error under intervention. Detailed proofs are provided in Appendix \ref{sec:appendix_regimes} and \ref{sec:proof_reduction}.

\subsection{Asymptotic Intervention Efficiency}
\label{subsec:theory_efficiency}

MCBM imposes a nested structure on the otherwise flat $O(K)$ intervention budget of standard CBM. We model intervention as a ``lazy verification'' workflow: inspect ranked concepts from coarse to fine, correct the current prefix, and stop once the prediction is fixed. The stopping level $\ell^*$ is minimal-sufficient, so the analysis does not require exhausting all concepts. We bound the expected cost $E = \mathbb{E}[N(x, y)]$ under two structural assumptions:

\begin{enumerate}
    \item \textbf{Geometric Concept Growth:} The number of concepts grows as $k_i = k_1 r^{i-1}$ (for rate $r > 1$).
    \item \textbf{Geometric Information Decay:} The probability of needing to intervene deeper decays as $P(\ell^*=i) \leq C \gamma^{i-1}$ (for decay $\gamma \in (0,1)$).
\end{enumerate}

The efficiency of MCBM is determined by the interplay between the model's capacity expansion ($r$) and its information capture rate ($1/\gamma$).
Empirically, this assumption is testable by measuring the distribution of $\ell^*$ over initially misclassified samples; on CUB, the observed mass concentrates on early levels and follows an exponential trend with fitted $\gamma=0.9636$ (Appendix~\ref{sec:app_geometric_decay}).

\begin{theorem}[Intervention Efficiency Regimes]
\label{thm:efficiency_regimes}
Let $\rho = r\gamma$ be the spectral ratio between concept growth and probability decay. The expected intervention cost $E$ scales with the total concept count $K$ according to three distinct regimes:
\begin{equation}
E = \begin{cases} 
\Theta(1) & \text{if } \gamma < 1/r \quad \text{(Efficient Regime)} \\
\Theta(\log K) & \text{if } \gamma = 1/r \quad \text{(Balanced Regime)} \\
\Theta(K^\alpha) & \text{if } \gamma > 1/r \quad \text{(Heavy-Tailed Regime)}
\end{cases}
\end{equation}
where $\alpha = 1 + \log_r \gamma > 0$.
\end{theorem}

\begin{remark}
\textbf{(The Spectral Race)} This theorem frames efficiency as a ``race'' between adding costs and gaining information.
\begin{itemize}
    \item \textbf{Balanced ($\gamma = 1/r$):} costs scale as $\Theta(\log K)$, \textbf{formally reducing human effort from linear $O(K)$ to logarithmic order.}
    \item \textbf{Efficient ($\gamma < 1/r$):} information capture outpaces growth; costs converge to a constant $\Theta(1)$, surpassing the logarithmic guarantee.
    \item \textbf{Heavy-Tailed ($\gamma > 1/r$):} unordered concepts degrade to polynomial scaling, underscoring the necessity of the mRMR prior.
\end{itemize}
\end{remark}

\subsection{Error Reduction Guarantee via Hellman-Raviv}
\label{subsec:theory_correctness}

While classical analysis often uses Fano's Inequality to establish lower bounds on error, we are interested in a sufficiency condition: does intervening on MCBM concepts \textit{guarantee} a reduction in error? We utilize the Hellman-Raviv Inequality to derive an upper bound on the classification error $P_e(k)$ after $k$ interventions.


\begin{theorem}[Upper Bound on Intervention Error]
\label{thm:error_bound_main}
Let $\epsilon$ be the intervention shift penalty, defined by the Kullback-Leibler divergence between the training distribution (soft concepts) and the intervention distribution (hard concepts). The probability of error after $k$ interventions is upper-bounded by:
\begin{equation}
P_e(k) \leq \frac{1}{2} \left( H(Y) - I(Y; \tilde{C}^{(k)}) \right) + \sqrt{\frac{\epsilon}{2}},
\end{equation}
where $H(Y)$ is the marginal entropy of the labels and $I(Y; \tilde{C}^{(k)})$ is the mutual information between the labels and the intervened concept vector.
\end{theorem}

\begin{remark}
\textbf{(Justification of Matryoshka Objective)} This theorem provides a rigorous justification for our training objective. To minimize the error bound $P_e(k)$, one must maximize the mutual information $I(Y; \tilde{C}^{(k)})$. Unlike standard CBM that only maximize this for the full vector ($k=K$), the Matryoshka loss explicitly maximizes this mutual information for every prefix $k \in \mathcal{D}$. This guarantees that the error bound tightens monotonically with every additional intervention step.
\end{remark}

\begin{remark}
\textbf{(Practical Non-Idealities)} Deep CBMs only approximate the Bayesian classifier used in the Hellman--Raviv argument, and joint training can introduce concept leakage~\cite{parisini2025leakage}. These gaps are absorbed by $\epsilon$: severe leakage, overfitting, or soft-to-hard intervention shift inflates the KL penalty and loosens the bound. A strictly sequential MCBM that isolates $X\!\to\!C$ from $C\!\to\!Y$ still preserves strong low-budget accuracy in our experiments.
\end{remark}
\section{Experiments}
\label{sec:experiments}

We evaluate MCBM on three fine-grained benchmarks, CUB-200-2011~\cite{wah2011caltech}, LAD~\cite{zhao2019lad}, and CelebA~\cite{liu2015celeba}, around four research questions. \textbf{RQ1:} Does a single MCBM match independently trained models across concept granularities? \textbf{RQ2:} Does mRMR ordering reduce the human effort of test-time intervention? \textbf{RQ3:} Does MCBM offer a smooth annotation-accuracy trade-off without retraining? \textbf{RQ4:} How effectively does mRMR concentrate discriminative information in early dimensions versus random or heuristic orderings?


\subsection{Experimental Setup}

\textbf{Datasets and Metrics.} We use \textbf{CUB} ($K=112$)~\cite{wah2011caltech}, \textbf{CelebA} ($K=40$)~\cite{liu2015celeba}, and \textbf{LAD} ($K=256$)~\cite{zhao2019lad}. CelebA is treated as a binary attribute-based benchmark with \textit{Attractive} as the target and the remaining facial attributes as concepts. We report \textit{Accuracy} and \textit{Macro F1}, and \textit{Accuracy@k} (accuracy after correcting the top-$k$ concepts) for intervention efficiency.

\textbf{Baselines.} We compare MCBM against: (1) \textbf{Independent CBM}, trained individually for each bottleneck size; (2) \textbf{Sequential CBM}, which isolates concept learning from label prediction; (3) \textbf{Sparse CBM} and \textbf{Label-free CBM}~\cite{oikarinen2023labelfree}; (4) \textbf{Random-Order MCBM}, to ablate the efficacy of mRMR ranking; and (5) \textbf{Efficient MCBM}, our weight-tying variant for memory-constrained scenarios. Table~\ref{tab:cub_baselines} reports CUB baselines under the official split.

\textbf{Implementation.} The default backbone is pre-trained Inception\_v3 (Sec.~\ref{subsec:backbone}). ResNet-50, Inception\_v3, and ViT-B/16 are fine-tuned end-to-end from standard pre-trained weights; CLIP-ViT-B/32 is kept frozen as a zero-shot baseline. Training uses the joint Matryoshka loss (Eq.~2) with dataset-specific nesting dimensions $\mathcal{D}$.

\subsection{Matryoshka Concept Learning Efficacy}
\label{subsec:performance}

We first analyze hierarchical representation learning within a single training run. Table~\ref{tab:mcbm_efficient_comparison} compares \textit{Standard} (independent heads) versus \textit{Efficient MCBM} (weight-tying), while Figure~\ref{fig:val_acc_per_head} visualizes training dynamics.

\begin{figure}[h]
    \centering
    \begin{minipage}[t]{0.39\linewidth}
        \centering
        \begin{minipage}[t][0.17\textheight][c]{\linewidth}
            \centering
            \resizebox{0.92\linewidth}{!}{%
            \scriptsize
            \setlength{\tabcolsep}{4pt}
            \begin{tabular}{lc}
            \toprule
            \textbf{Method} & \textbf{Test Acc. (\%)} \\
            \midrule
            Sparse CBM & 53.75 \\
            Sequential CBM & 69.87 \\
            Independent CBM & 70.38 \\
            Label-free CBM~\cite{oikarinen2023labelfree} & 74.06 \\
            MCBM (Standard, $K=112$) & 73.20 \\
            Efficient MCBM ($K=112$) & 72.94 \\
            \bottomrule
            \end{tabular}%
            }
        \end{minipage}
        \begin{minipage}[t][0.06\textheight][t]{\linewidth}
            \captionof{table}{\textbf{CUB baseline comparison.} Test accuracy under the official split.}
            \label{tab:cub_baselines}
        \end{minipage}
    \end{minipage}
    \hfill
    \begin{minipage}[t]{0.58\linewidth}
        \centering
        \begin{minipage}[t][0.17\textheight][c]{\linewidth}
            \centering
            \includegraphics[width=\linewidth,height=0.17\textheight,keepaspectratio]{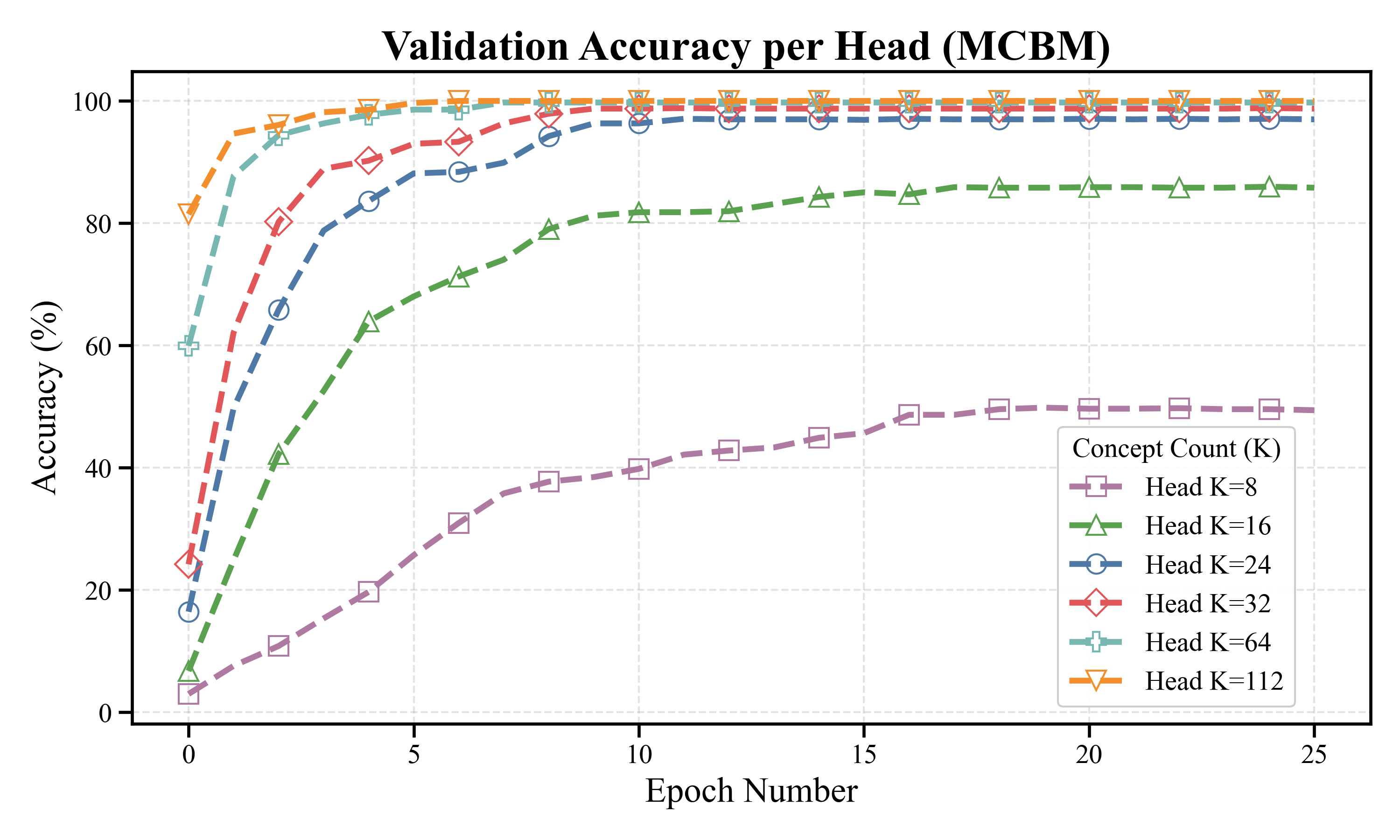}
        \end{minipage}
        \begin{minipage}[t][0.06\textheight][t]{\linewidth}
            \captionof{figure}{\textbf{Validation accuracy per head (CUB).} Compressed heads rapidly approach the full model.}
            \label{fig:val_acc_per_head}
        \end{minipage}
    \end{minipage}
\end{figure}


\textbf{Task-Dependent Saturation and Competitive Performance (RQ1).}
Table~\ref{tab:mcbm_efficient_comparison} reveals distinct saturation patterns: CUB and CelebA plateau early at $K=32$ and $K=8$, while LAD requires up to $K=256$, so a single MCBM can be aggressively truncated for simple tasks and kept full-capacity for complex ones. Figure~\ref{fig:val_acc_per_head} further shows stable simultaneous convergence across nesting levels, with compressed heads (e.g., $K=32$) rapidly approaching the full-capacity head ($K=112$); Efficient MCBM preserves this behavior with only a $\sim$1\% drop on LAD, matching independently trained models without maintaining multiple architectures.

\subsection{Efficient Concept Intervention}
\label{subsec:intervention}

A primary contribution of MCBM is minimizing the cost of human-in-the-loop intervention. We simulate a test-time scenario where a user sequentially corrects mispredicted concepts. Figure~\ref{fig:intervention_efficiency} compares the efficacy of our \textit{mRMR-based ordering} against a \textit{Random ordering} baseline across CUB, CelebA, and LAD.



\textbf{Efficiency and Cognitive Load (RQ2).} 
mRMR ordering lowers intervention cost by an order of magnitude versus unordered baselines. On CelebA, MCBM reaches near-perfect accuracy after merely \textbf{3 interventions}, rendering the remaining 37 concepts effectively redundant and validating a ``lazy verification'' strategy. The advantage persists in complex ``threshold regimes'' such as LAD ($K=256$), where mRMR initiates error recovery much earlier (around $k=125$) than random selection, prioritizing discriminative signals regardless of task complexity.

\begin{figure}[ht]

    \centering
    \begin{minipage}[b]{0.51\linewidth}
        \centering
        \resizebox{\linewidth}{!}{%
            \scriptsize
            \setlength{\tabcolsep}{4pt}
            \begin{tabular}{l|cc|cc}
                \toprule
                \multirow{2}{*}{\textbf{Head Size ($K$)}}
                    & \multicolumn{2}{c|}{\textbf{MCBM (Standard)}}
                    & \multicolumn{2}{c}{\textbf{Efficient MCBM}} \\
                & \textbf{Acc (↑,\%)} & \textbf{F1 (↑,\%)}
                & \textbf{Acc (↑,\%)} & \textbf{F1 (↑,\%)} \\
                \midrule
                \multicolumn{5}{l}{\textit{\textbf{CUB Dataset}}} \\
                \midrule
                $K=8$   & 40.94 & 34.37 & 38.30 & 32.23 \\
                $K=16$  & 63.24 & 61.39 & 61.70 & 59.93 \\
                \rowcolor{rowblue} $K=24$  & 70.33 & 69.75 & 69.31 & 69.94 \\
                \rowcolor{rowblue} $K=32$  & 71.80 & 71.73 & 71.30 & 71.33 \\
                \rowcolor{rowblue} $K=64$  & 73.21 & 73.19 & 72.68 & 72.83 \\
                \rowcolor{rowblue} $K=112$ & 73.20 & 73.35 & 72.94 & 73.10 \\
                \midrule
                \multicolumn{5}{l}{\textit{\textbf{LAD Dataset}}} \\
                \midrule
                $K=8$   & 7.76  & 2.02  & 7.86  & 4.07  \\
                $K=16$  & 14.64 & 6.62  & 14.08 & 8.49  \\
                $K=24$  & 22.22 & 14.07 & 21.22 & 13.97 \\
                $K=32$  & 35.40 & 26.25 & 35.09 & 24.57 \\
                $K=64$  & 67.34 & 57.64 & 66.49 & 53.03 \\
                \rowcolor{rowblue} $K=128$ & 94.62 & 93.76 & 94.15 & 93.88 \\
                \rowcolor{rowblue} $K=256$ & 97.89 & 97.76 & 97.14 & 94.91 \\
                \midrule
                \multicolumn{5}{l}{\textit{\textbf{CelebA Dataset}}} \\
                \midrule
                \rowcolor{rowblue} $K=8$  & 79.40 & 79.28 & 79.16 & 79.13 \\
                \rowcolor{rowblue} $K=16$ & 79.50 & 79.41 & 79.30 & 79.18 \\
                \rowcolor{rowblue} $K=24$ & 79.52 & 79.45 & 79.26 & 79.13 \\
                \rowcolor{rowblue} $K=32$ & 79.52 & 79.46 & 79.23 & 79.09 \\
                \rowcolor{rowblue} $K=40$ & 79.58 & 79.51 & 79.11 & 78.94 \\
                \bottomrule
            \end{tabular}%
        }
        \captionof{table}{\textbf{Performance comparison on CUB, LAD, and CelebA.}
        Accuracy and F1 of Standard and Efficient MCBM across nesting dimensions $K$.
        Blue rows highlight the saturation regime where additional heads yield diminishing returns,
        indicating that a much smaller $K$ already captures most of the achievable accuracy.}
        \label{tab:mcbm_efficient_comparison}
    \end{minipage}
    \hfill
    \begin{minipage}[b]{0.48\linewidth}
        \centering
        \includegraphics[width=\linewidth]{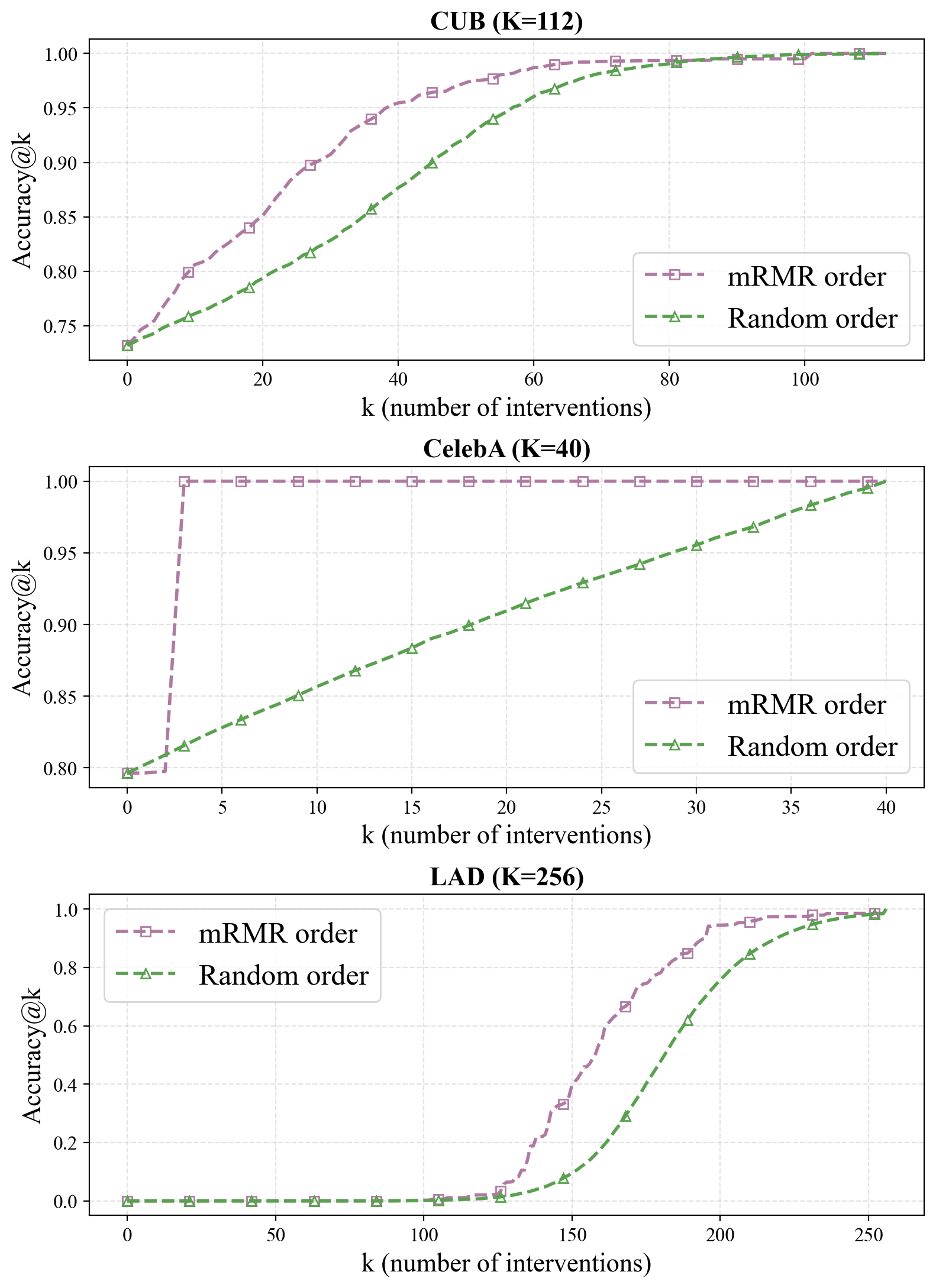}
        \captionof{figure}{\textbf{Intervention efficiency across datasets.}
        Accuracy@$k$ vs.\ intervention count $k$; mRMR ordering consistently dominates random ordering.}
        \label{fig:intervention_efficiency}
    \end{minipage}
    \end{figure}

\textbf{Dynamic Flexibility and Trade-offs (RQ3).} 
The steep recovery curves make MCBM an ``any-time'' intervention system: practitioners can adjust intervention depth at test time without retraining, stopping early to secure most of the gains or continuing for higher precision, all within one unified model.



\textbf{Leakage, Stability, and Matched Heads.}
Additional CUB diagnostics support that the gains are due to semantic ordering rather than shortcut leakage. The top-$K$ concept predictor remains accurate across budgets (92.44\% at $K=8$, 93.14\% at $K=16$, 93.80\% at $K=32$, and 96.27\% at $K=112$). A strictly Sequential MCBM, which blocks task gradients from reaching the concept extractor, still achieves 72.09\% accuracy at $K=16$ and 73.75\% at $K=112$. The mRMR hierarchy is also stable: across five random seeds, the selected top-$K$ sets for $K\in\{8,16,24,32\}$ have IoU $=1.00$, and a full-head weight-based ordering strongly correlates with mRMR (Spearman $=0.934$) while yielding a near-identical recovery curve. Finally, using the full $K=112$ head for partial interventions is inferior to matched Matryoshka heads, dropping from 88.71\% to 83.03\% at $k=16$ and from 98.46\% to 91.92\% at $k=32$, due to soft/hard concept distribution shift. Full ablations are in Appendix~\ref{sec:app_rebuttal_experiments}.

\subsection{Impact of Backbone Architecture}
\label{subsec:backbone}

To assess the generality of MCBM, we evaluate its performance when scaling the feature extractor across different architectural paradigms. We compare ResNet-18, ResNet-50, Inception\_v3, Vision Transformers (ViT-B/16), and a pre-trained Vision-Language model (CLIP-ViT-B/32). Figure~\ref{fig:backbone_comparison}(a) reports the F1 Score across various nesting levels on the CUB dataset.

\textbf{Architecture Agnosticism and Inception\_v3 Advantage.}
The ``Matryoshka property''---monotonic improvement with concept capacity---holds across CNNs and Transformers, confirming that MCBM is compatible with modern vision backbones. Among them, Inception\_v3 consistently outperforms ResNet-50 and ViT by roughly 5--10\% F1, and the gap is most pronounced at low granularities ($K=8, 16$), suggesting that its multi-scale inception modules align well with the hierarchical nature of fine-grained concept detection.

\textbf{Limitations of Zero-Shot Features.}
The frozen CLIP backbone underperforms on this fine-grained task: while vision-language pre-training offers generality, supervised fine-tuning remains essential for capturing subtle attributes such as ``beak shape''.

\subsection{Impact of Ranking Strategy on Performance}
\label{subsec:ranking_impact}

While Section \ref{subsec:intervention} demonstrated the benefit of mRMR for \textit{human intervention}, we now investigate its impact on \textit{model representation learning}. We compare two MCBM models trained identically on CUB, differing only in their concept ordering: one uses our proposed mRMR schedule, while the other uses a fixed Random order. Figure~\ref{fig:ranking_impact}(b) visualizes the F1 scores across nesting levels.

\textbf{The Necessity of mRMR for Compression.}
The results reveal a stark contrast in information density. At low dimensions, the mRMR-based model dominates the random baseline. For instance, at $K=8$, the Random model collapses to near-chance performance with approximately 3\% F1 score, whereas the mRMR model retains a respectable F1 score of nearly 34\%. Similarly, at $K=16$, mRMR outperforms Random by nearly 40 percentage points. This disparity confirms that the ``Matryoshka property'', which denotes the ability to perform valid inference at early exits, does not emerge spontaneously from the joint loss alone. It requires the mRMR prior to explicitly guide the compression of semantic information into the earliest dimensions.

As concepts approach the full set ($K=112$), the gap closes and both methods reach parity ($\sim$75\% F1), since ordering becomes irrelevant once the full vocabulary is available---reinforcing that the value of mRMR lies in elasticity, not in peak accuracy.

\begin{figure}[ht]
    \centering
    \begin{minipage}{0.49\linewidth}
        \centering
        \textbf{(a)}\\[-2pt]
        \includegraphics[width=\linewidth]{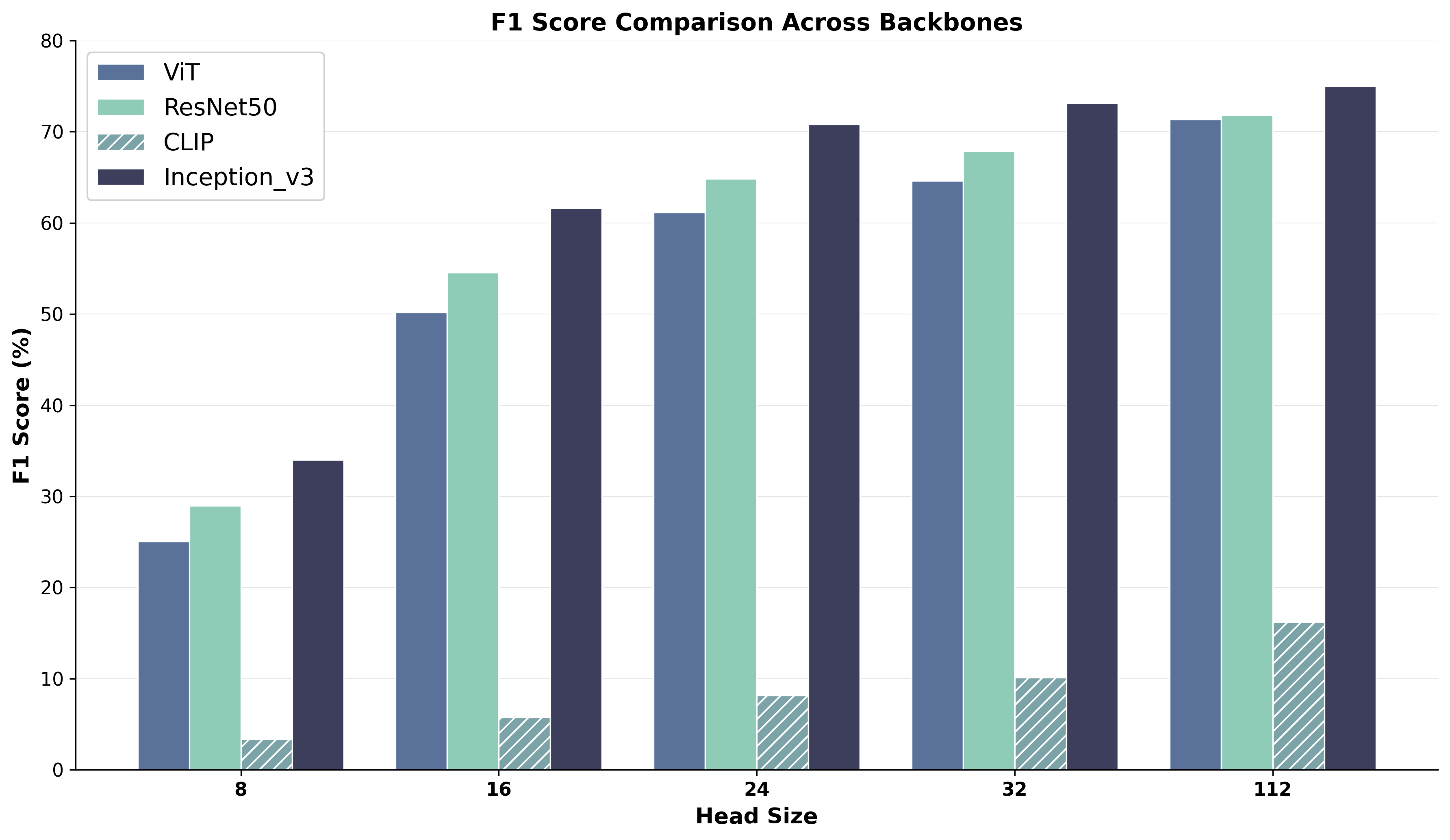}
    \end{minipage}
    \hfill
    \begin{minipage}{0.49\linewidth}
        \centering
        \textbf{(b)}\\[-2pt]
        \includegraphics[width=\linewidth]{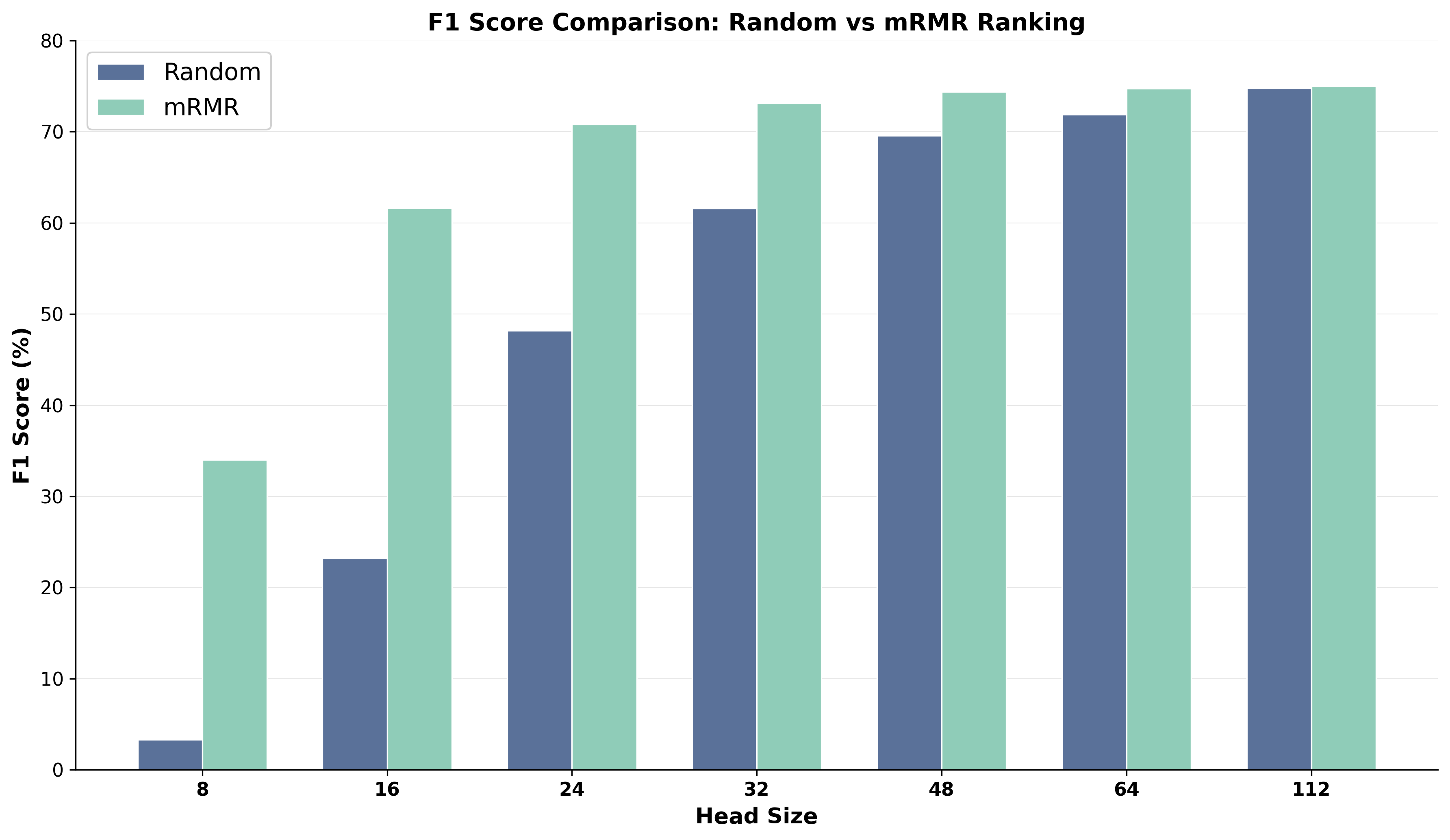}
    \end{minipage}
    \caption{\textbf{Backbone and Ranking Effects.} (a) Inception\_v3 consistently yields the highest F1, while frozen CLIP struggles with fine-grained attributes. (b) mRMR significantly outperforms random ordering at low dimensions, enabling effective model compression.}
    \label{fig:backbone_comparison}
    \label{fig:ranking_impact}
\end{figure}


\textbf{Information Concentration (RQ4).} 
mRMR is the structural prior that makes Matryoshka compression usable: the catastrophic collapse of the random baseline at $K=8$ is a counterfactual proof that joint loss alone does not hierarchically organize the latent space. By front-loading the most discriminative semantic signals, mRMR turns the bottleneck from a flat representation into a strictly ordered hierarchy that both recognizes concepts and encodes their relative importance.

\section{Conclusion}
We introduce the Matryoshka Concept Bottleneck Models (MCBM), which imposes a nested, information-dense hierarchy on concepts via mRMR. Theoretically, we characterize intervention as a ``spectral race'', proving a reduction in human verification costs from linear to logarithmic scales. Empirically, MCBM achieves performance parity with independent baselines while enabling flexible, elastic deployment. Our work effectively bridges the gap between structural efficiency and semantic interpretability, establishing a new standard for scalable human-in-the-loop AI.

\bibliographystyle{abbrv}
\bibliography{main}

@InProceedings{koh2020concept,
  title = 	 {Concept Bottleneck Models},
  author =       {Koh, Pang Wei and Nguyen, Thao and Tang, Yew Siang and Mussmann, Stephen and Pierson, Emma and Kim, Been and Liang, Percy},
  booktitle = 	 {Proceedings of the 37th International Conference on Machine Learning},
  pages = 	 {5338--5348},
  year = 	 {2020},
  editor = 	 {III, Hal Daumé and Singh, Aarti},
  volume = 	 {119},
  series = 	 {Proceedings of Machine Learning Research},
  month = 	 {13--18 Jul},
  publisher =    {PMLR},
}

@inproceedings{kusupati2022matryoshka,
 author = {Kusupati, Aditya and Bhatt, Gantavya and Rege, Aniket and Wallingford, Matthew and Sinha, Aditya and Ramanujan, Vivek and Howard-Snyder, William and Chen, Kaifeng and Kakade, Sham and Jain, Prateek and Farhadi, Ali},
 booktitle = {Advances in Neural Information Processing Systems},
 editor = {S. Koyejo and S. Mohamed and A. Agarwal and D. Belgrave and K. Cho and A. Oh},
 pages = {30233--30249},
 publisher = {Curran Associates, Inc.},
 title = {Matryoshka Representation Learning},
 url = {https://proceedings.neurips.cc/paper_files/paper/2022/file/c32319f4868da7613d78af9993100e42-Paper-Conference.pdf},
 volume = {35},
 year = {2022}
}

@InProceedings{shin2023closer,
  title = 	 {A Closer Look at the Intervention Procedure of Concept Bottleneck Models},
  author =       {Shin, Sungbin and Jo, Yohan and Ahn, Sungsoo and Lee, Namhoon},
  booktitle = 	 {Proceedings of the 40th International Conference on Machine Learning},
  pages = 	 {31504--31520},
  year = 	 {2023},
  editor = 	 {Krause, Andreas and Brunskill, Emma and Cho, Kyunghyun and Engelhardt, Barbara and Sabato, Sivan and Scarlett, Jonathan},
  volume = 	 {202},
  series = 	 {Proceedings of Machine Learning Research},
  month = 	 {23--29 Jul},
  publisher =    {PMLR},
  url = 	 {https://proceedings.mlr.press/v202/shin23a.html},
}

@inproceedings{
oikarinen2023labelfree,
title={Label-free Concept Bottleneck Models},
author={Tuomas Oikarinen and Subhro Das and Lam M. Nguyen and Tsui-Wei Weng},
booktitle={The Eleventh International Conference on Learning Representations },
year={2023},
url={https://openreview.net/forum?id=FlCg47MNvBA}
}

@inproceedings{hueditable,
  title={Editable Concept Bottleneck Models},
  author={Hu, Lijie and Ren, Chenyang and Hu, Zhengyu and Lin, Hongbin and Wang, Cheng-Long and Tan, Zhen and Lyu, Weimin and Zhang, Jingfeng and Xiong, Hui and Wang, Di},
  booktitle={Forty-second International Conference on Machine Learning},
  year={2025}
}

@inproceedings{
yuksekgonul2022post,
title={Post-hoc Concept Bottleneck Models},
author={Mert Yuksekgonul and Maggie Wang and James Zou},
booktitle={The Eleventh International Conference on Learning Representations },
year={2023},
url={https://openreview.net/forum?id=nA5AZ8CEyow}
}

@techreport{wah2011caltech,
    title = {The Caltech-UCSD Birds-200-2011 Dataset},
    Author = {Wah, C. and Branson, S. and Welinder, P. and Perona, P. and Belongie, S.},
    Year = {2011},
    Institution = {California Institute of Technology},
    Number = {CNS-TR-2011-001},
}

@InProceedings{chowdhury2024adacbm,
    author = { Chowdhury, Townim F. and Phan, Vu Minh Hieu and Liao, Kewen and To, Minh-Son and Xie, Yutong and van den Hengel, Anton and Verjans, Johan W. and Liao, Zhibin},
    title = { { AdaCBM: An Adaptive Concept Bottleneck Model for Explainable and Accurate Diagnosis } },
    booktitle = {proceedings of Medical Image Computing and Computer Assisted Intervention -- MICCAI 2024},
    year = {2024},
    publisher = {Springer Nature Switzerland},
    volume = {LNCS 15010},
    month = {October},
    page = {35 -- 45}
}

@inproceedings{ahmad2018interpretable,
author = {Ahmad, Muhammad Aurangzeb and Eckert, Carly and Teredesai, Ankur},
title = {Interpretable Machine Learning in Healthcare},
year = {2018},
isbn = {9781450357944},
publisher = {Association for Computing Machinery},
address = {New York, NY, USA},
url = {https://doi.org/10.1145/3233547.3233667},
doi = {10.1145/3233547.3233667},
booktitle = {Proceedings of the 2018 ACM International Conference on Bioinformatics, Computational Biology, and Health Informatics},
pages = {559–560},
numpages = {2},
location = {Washington, DC, USA},
series = {BCB '18}
}

@InProceedings{yu2018artificial,
author = {Yu, Jiahui and Lin, Zhe and Yang, Jimei and Shen, Xiaohui and Lu, Xin and Huang, Thomas S.},
title = {Generative Image Inpainting With Contextual Attention},
booktitle = {Proceedings of the IEEE Conference on Computer Vision and Pattern Recognition (CVPR)},
month = {June},
year = {2018}
}

@article{daneshjou2021dermatologist,
   title={Disparities in dermatology AI performance on a diverse, curated clinical image set},
   volume={8},
   ISSN={2375-2548},
   url={http://dx.doi.org/10.1126/sciadv.abq6147},
   DOI={10.1126/sciadv.abq6147},
   number={32},
   journal={Science Advances},
   publisher={American Association for the Advancement of Science (AAAS)},
   author={Daneshjou, Roxana and Vodrahalli, Kailas and Novoa, Roberto A. and Jenkins, Melissa and Liang, Weixin and Rotemberg, Veronica and Ko, Justin and Swetter, Susan M. and Bailey, Elizabeth E. and Gevaert, Olivier and Mukherjee, Pritam and Phung, Michelle and Yekrang, Kiana and Fong, Bradley and Sahasrabudhe, Rachna and Allerup, Johan A. C. and Okata-Karigane, Utako and Zou, James and Chiou, Albert S.},
   year={2022},
   month={aug }
}

@inproceedings{havasi2022addressing,
author = {Havasi, Marton and Parbhoo, Sonali and Doshi-Velez, Finale},
title = {Addressing leakage in concept bottleneck models},
year = {2022},
isbn = {9781713871088},
publisher = {Curran Associates Inc.},
address = {Red Hook, NY, USA},
booktitle = {Proceedings of the 36th International Conference on Neural Information Processing Systems},
articleno = {1699},
numpages = {12},
location = {New Orleans, LA, USA},
series = {NIPS '22}
}

@inproceedings{zarlenga2023learning,
author = {Zarlenga, Mateo Espinosa and Collins, Katherine M. and Dvijotham, Krishnamurthy (Dj) and Weller, Adrian and Shams, Zohreh and Jamnik, Mateja},
title = {Learning to receive help: intervention-aware concept embedding models},
year = {2023},
publisher = {Curran Associates Inc.},
address = {Red Hook, NY, USA},
booktitle = {Proceedings of the 37th International Conference on Neural Information Processing Systems},
articleno = {1648},
numpages = {27},
location = {New Orleans, LA, USA},
series = {NIPS '23}
}

@inproceedings{antognini2021concept,
    title = "Rationalization through Concepts",
    author = "Antognini, Diego  and
      Faltings, Boi",
    editor = "Zong, Chengqing  and
      Xia, Fei  and
      Li, Wenjie  and
      Navigli, Roberto",
    booktitle = "Findings of the Association for Computational Linguistics: ACL-IJCNLP 2021",
    month = aug,
    year = "2021",
    address = "Online",
    publisher = "Association for Computational Linguistics",
    url = "https://aclanthology.org/2021.findings-acl.68/",
    doi = "10.18653/v1/2021.findings-acl.68",
    pages = "761--775"
}

@InProceedings{wong2021leveraging,
  title = 	 {Leveraging Sparse Linear Layers for Debuggable Deep Networks},
  author =       {Wong, Eric and Santurkar, Shibani and Madry, Aleksander},
  booktitle = 	 {Proceedings of the 38th International Conference on Machine Learning},
  pages = 	 {11205--11216},
  year = 	 {2021},
  editor = 	 {Meila, Marina and Zhang, Tong},
  volume = 	 {139},
  series = 	 {Proceedings of Machine Learning Research},
  month = 	 {18--24 Jul},
  publisher =    {PMLR},
  url = 	 {https://proceedings.mlr.press/v139/wong21b.html},
}

@INPROCEEDINGS{ding2005minimum,
  author={Ding, C. and Peng, H.},
  booktitle={Computational Systems Bioinformatics. CSB2003. Proceedings of the 2003 IEEE Bioinformatics Conference. CSB2003}, 
  title={Minimum redundancy feature selection from microarray gene expression data}, 
  year={2003},
  volume={},
  number={},
  pages={523-528},
  doi={10.1109/CSB.2003.1227396}}

@inproceedings{yang2023language,
  title={Language in a bottle: Language model guided concept bottlenecks for interpretable image classification},
  author={Yang, Yue and Panagopoulou, Artemis and Zhou, Shenghao and Jin, Daniel and Callison-Burch, Chris and Yatskar, Mark},
  booktitle={Proceedings of the IEEE/CVF Conference on Computer Vision and Pattern Recognition},
  pages={19187--19197},
  year={2023}
}

@misc{semenov2024sparse,
      title={Sparse Concept Bottleneck Models: Gumbel Tricks in Contrastive Learning}, 
      author={Andrei Semenov and Vladimir Ivanov and Aleksandr Beznosikov and Alexander Gasnikov},
      year={2024},
      eprint={2404.03323},
      archivePrefix={arXiv},
      primaryClass={cs.CV},
      url={https://arxiv.org/abs/2404.03323}, 
}

@ARTICLE{peng2005feature,
  author={Hanchuan Peng and Fuhui Long and Ding, C.},
  journal={IEEE Transactions on Pattern Analysis and Machine Intelligence}, 
  title={Feature selection based on mutual information criteria of max-dependency, max-relevance, and min-redundancy}, 
  year={2005},
  volume={27},
  number={8},
  pages={1226-1238},
  doi={10.1109/TPAMI.2005.159}}

@inproceedings{zhao2019lad,
  title={A Large-scale Attribute Dataset for Zero-shot Learning},
  author={Zhao, Bo and Fu, Yanwei and Liang, Rui and Wu, Jiahong and Wang, Yonggang and Wang, Yizhou},
  booktitle={Proceedings of the IEEE Conference on Computer Vision and Pattern Recognition Workshops},
  year={2019}
}

@inproceedings{liu2015celeba,
  title = {Deep Learning Face Attributes in the Wild},
  author = {Liu, Ziwei and Luo, Ping and Wang, Xiaogang and Tang, Xiaoou},
  booktitle = {Proceedings of International Conference on Computer Vision (ICCV)},
  month = {December},
  year = {2015} 
}

@article{lin2026controllable,
  title={Controllable Concept Bottleneck Models},
  author={Lin, Hongbin and Ren, Chenyang and Xu, Juangui and Hu, Zhengyu and Wang, Cheng-Long and Shu, Yao and Xiong, Hui and Zhang, Jingfeng and Wang, Di and Hu, Lijie},
  journal={arXiv preprint arXiv:2601.00451},
  year={2026}
}

@inproceedings{Chauhan2022,
author = {Chauhan, Kushal and Tiwari, Rishabh and Freyberg, Jan and Shenoy, Pradeep and Dvijotham, Krishnamurthy},
title = {Interactive concept bottleneck models},
year = {2023},
isbn = {978-1-57735-880-0},
publisher = {AAAI Press},
url = {https://doi.org/10.1609/aaai.v37i5.25736},
doi = {10.1609/aaai.v37i5.25736},
booktitle = {Proceedings of the Thirty-Seventh AAAI Conference on Artificial Intelligence and Thirty-Fifth Conference on Innovative Applications of Artificial Intelligence and Thirteenth Symposium on Educational Advances in Artificial Intelligence},
articleno = {667},
numpages = {8},
series = {AAAI'23/IAAI'23/EAAI'23}
}

@inproceedings{ECBM,
      title={Energy-Based Concept Bottleneck Models: Unifying Prediction, Concept Intervention, and Probabilistic Interpretations}, 
      author={Xu, Xinyue and Qin, Yi and Mi, Lu and Wang, Hao and Li, Xiaomeng},
      booktitle={International Conference on Learning Representations},
      year={2024}
}

@inproceedings{vandenhirtz2024stochastic,
  title={Stochastic Concept Bottleneck Models},
  author={Vandenhirtz, Moritz and Laguna, Sonia and Marcinkevi{\v{c}}s, Ri{\v{c}}ards and Vogt, Julia E.},
  booktitle={Advances in Neural Information Processing Systems},
  year={2024},
  url={https://arxiv.org/abs/2406.19272}
}

@article{defelice2026causal,
  title={Causally Reliable Concept Bottleneck Models},
  author={De Felice, Giovanni and Casanova Flores, Arianna and De Santis, Francesco and Santini, Silvia and Schneider, Johannes and Barbiero, Pietro and Termine, Alberto},
  journal={arXiv preprint arXiv:2503.04363},
  year={2026},
  url={https://arxiv.org/abs/2503.04363}
}

@inproceedings{bhan2025complete,
  title={Towards Achieving Concept Completeness for Textual Concept Bottleneck Models},
  author={Bhan, Milan and Choho, Yann and Vittaut, Jean-No{\"e}l and Chesneau, Nicolas and Moreau, Pierre and Lesot, Marie-Jeanne},
  booktitle={Findings of the Association for Computational Linguistics: EMNLP 2025},
  pages={2007--2024},
  year={2025},
  publisher={Association for Computational Linguistics},
  url={https://aclanthology.org/2025.findings-emnlp.106/}
}

@article{parisini2025leakage,
  title={Leakage and interpretability in concept-based models},
  author={Parisini, Edoardo and Chakraborti, Tathagata and Harbron, Chris and MacArthur, Ben D. and Banerji, Christopher R. S.},
  journal={arXiv preprint arXiv:2504.14094},
  year={2025},
  url={https://arxiv.org/abs/2504.14094}
}

@article{hu2024towards,
  title={Towards multi-dimensional explanation alignment for medical classification},
  author={Hu, Lijie and Lai, Songning and Chen, Wenshuo and Xiao, Hongru and Lin, Hongbin and Yu, Lu and Zhang, Jingfeng and Wang, Di},
  journal={Advances in Neural Information Processing Systems},
  volume={37},
  pages={129640--129671},
  year={2024}
}

@inproceedings{hu2025stable,
  title={Stable vision concept transformers for medical diagnosis},
  author={Hu, Lijie and Lai, Songning and Hua, Yuan and Yang, Shu and Zhang, Jingfeng and Wang, Di},
  booktitle={Joint European Conference on Machine Learning and Knowledge Discovery in Databases},
  pages={317--332},
  year={2025},
  organization={Springer}
}

@inproceedings{zhang2025locate,
  title={Locate-then-edit for Multi-hop Factual Recall under Knowledge Editing},
  author={Zhang, Zhuoran and Li, Yongxiang and Kan, Zijian and Cheng, Keyuan and Hu, Lijie and Wang, Di},
  booktitle={International Conference on Machine Learning},
  pages={75369--75391},
  year={2025},
  organization={PMLR}
}

@inproceedings{hu2025semi,
  title={Semi-supervised concept bottleneck models},
  author={Hu, Lijie and Huang, Tianhao and Xie, Huanyi and Gong, Xilin and Ren, Chenyang and Hu, Zhengyu and Yu, Lu and Ma, Ping and Wang, Di},
  booktitle={Proceedings of the IEEE/CVF International Conference on Computer Vision},
  pages={2110--2119},
  year={2025}
}

@inproceedings{lai2024faithful,
  title={Faithful vision-language interpretation via concept bottleneck models},
  author={Lai, Songning and Hu, Lijie and Wang, Junxiao and Berti-Equille, Laure and Wang, Di},
  booktitle={International Conference on Learning Representations},
  volume={2024},
  pages={23756--23779},
  year={2024}
}

@inproceedings{cheng2025compke,
  title={Compke: Complex question answering under knowledge editing},
  author={Cheng, Keyuan and Kan, Zijian and Zhang, Zhuoran and Ali, Muhammad Asif and Hu, Lijie and Wang, Di},
  booktitle={Findings of the Association for Computational Linguistics: ACL 2025},
  pages={2557--2576},
  year={2025}
}

@article{xu2025concept,
  title={Concept-based unsupervised domain adaptation},
  author={Xu, Xinyue and Hu, Yueying and Tang, Hui and Qin, Yi and Mi, Lu and Wang, Hao and Li, Xiaomeng},
  journal={arXiv preprint arXiv:2505.05195},
  year={2025}
}

@article{rafferty2026radiologist,
  title={Radiologist-Guided Causal Concept Bottleneck Models for Chest X-Ray Interpretation},
  author={Rafferty, Amy and Ramaesh, Rishi and Rajan, Ajitha},
  journal={arXiv preprint arXiv:2605.07785},
  year={2026}
}

@article{fokkema2026sample,
  title={Sample-efficient learning of concepts with theoretical guarantees: from data to concepts without interventions},
  author={Fokkema, Hidde and van Erven, Tim and Magliacane, Sara},
  journal={Advances in Neural Information Processing Systems},
  volume={38},
  pages={111783--111843},
  year={2026}
}

@article{kalampalikis2025towards,
  title={Towards reasonable concept bottleneck models},
  author={Kalampalikis, Nektarios and Gupta, Kavya and Vitanov, Georgi and Valera, Isabel},
  journal={arXiv preprint arXiv:2506.05014},
  year={2025}
}

@article{bussmann2025learning,
  title={Learning multi-level features with matryoshka sparse autoencoders},
  author={Bussmann, Bart and Nabeshima, Noa and Karvonen, Adam and Nanda, Neel},
  journal={arXiv preprint arXiv:2503.17547},
  year={2025}
}

@article{lai2026matryoshka,
  title={Matryoshka Representation Learning for Recommendation with Layer-and Hardness-Adaptive Negative Sampling},
  author={Lai, Riwei and Chen, Li and Chen, Weixin and Chen, Rui},
  journal={ACM Transactions on Intelligent Systems and Technology},
  volume={17},
  number={4},
  pages={1--25},
  year={2026},
  publisher={ACM New York, NY}
}


\appendix

\providecommand{\E}{\mathbb{E}}
\providecommand{\R}{\mathbb{R}}
\providecommand{\Prob}{\mathbb{P}}
\providecommand{\Order}[1]{\ensuremath{\Theta\!\left(#1\right)}}
\providecommand{\BigO}[1]{\ensuremath{\mathcal{O}\!\left(#1\right)}}
\providecommand{\Ind}{\mathbb{I}}

\numberwithin{equation}{section}


\section{Notation Table}
\label{app:nomenclature}

Table~\ref{tab:notation} summarizes the mathematical notations used throughout the paper.

\begin{table}[h]
\centering
\renewcommand{\arraystretch}{1.1}
\caption{Summary of Notations.}
\label{tab:notation}
\resizebox{0.7\linewidth}{!}{%
\scriptsize
\begin{tabular}{l p{7.5cm}}
\toprule
\textbf{Symbol} & \textbf{Description} \\
\midrule
\multicolumn{2}{l}{\textit{Inputs and Model Architecture}} \\
$\mathbf{x}, y$       & Input image and ground-truth target label \\
$K$                   & Total number of concepts \\
$C$                   & Number of target classes \\
$g_\theta(\cdot)$     & Feature extractor (Backbone) parameterized by $\theta$ \\
$\mathbf{l}$          & Raw, unordered concept logits vector, $\mathbf{l}\in\mathbb{R}^K$ \\
$\mathbf{c}_{GT}$     & Ground-truth concept binary vector \\
\midrule
\multicolumn{2}{l}{\textit{Matryoshka \& mRMR Optimization}} \\
$\Omega$                       & The set of all available concepts \\
$\mathcal{S}_t$                & The set of selected concepts at step $t$ (mRMR) \\
$I(X; Y)$                      & Mutual Information between variables $X$ and $Y$ \\
$\Pi(\cdot)$                   & Permutation module based on mRMR ranking \\
$\mathbf{c}_{\text{ord}}$      & Ordered concept vector after permutation \\
$\mathcal{D}$                  & Set of predefined nesting granularities (e.g., $\{8,16,\dots\}$) \\
$d$                            & Index for a specific nesting level / concept budget \\
$f_{\phi_d}(\cdot)$            & Prediction head for nesting level $d$ \\
$\mathbf{M}_d$                 & Binary mask for Efficient MCBM at level $d$ \\
$\mathcal{L}_{\text{total}}$   & Total joint optimization loss \\
\midrule
\multicolumn{2}{l}{\textit{Theoretical Analysis}} \\
$\ell^*$                  & Minimal sufficient intervention level for a specific instance \\
$k_i$                     & Number of concepts at geometric level $i$ \\
$r$                       & Geometric growth rate of concept set size ($r>1$) \\
$\gamma$                  & Geometric decay rate of information need ($\gamma\in(0,1)$) \\
$\rho$                    & Spectral ratio ($\rho=r\gamma$), determining efficiency regimes \\
$E$                       & Expected intervention cost \\
$P_e(k)$                  & Classification error probability after $k$ interventions \\
$\tilde{C}^{(k)}$         & Concept vector intervened up to index $k$ \\
$\epsilon$                & Intervention shift penalty (KL divergence) \\
\bottomrule
\end{tabular}%
}
\end{table}

\section{Proof of Intervention Efficiency Regimes}
\label{sec:appendix_regimes}

\noindent Building upon the intervention reduction established in the previous section, we now provide a rigorous asymptotic analysis of the expected intervention cost $E$ as a function of the total number of concepts $K$. While the previous section demonstrated that $E<K$, this section categorizes the exact scaling behavior of $E$. We show that the efficiency of the Matryoshka CBM is determined by a \emph{spectral race} between the growth of the concept sets (model capacity) and the decay of conditional entropy (information capture).

\subsection{Setup and Assumptions}

\noindent We analyze the behavior of the expected intervention cost $E=\sum_{i=1}^{L} k_i\cdot P(\ell^*=i)$ under the structural constraints imposed by the Matryoshka architecture.

\begin{assumption}[Geometric Concept Growth]
\label{ass:growth}
The Matryoshka nesting structure is constructed such that the size of the concept set at level $i$ grows geometrically with rate $r>1$:
\begin{equation}
k_i = k_1\cdot r^{i-1},\quad i=1,\dots,L.
\end{equation}
Given $k_L=K$, the number of levels scales logarithmically:
\begin{equation}
L=\log_r\!\left(\frac{K}{k_1}\right)+1=\Theta(\log K).
\end{equation}
\end{assumption}

\begin{assumption}[Geometric Probability Decay]
\label{ass:decay}
The probability mass function of the minimal sufficient level $\ell^*$ decays geometrically with rate $\gamma\in(0,1)$:
\begin{equation}
P(\ell^*=i)\le C\cdot \gamma^{i-1},
\end{equation}
where $C$ is a normalization constant ensuring $\sum P(\ell^*=i)=1$. The parameter $\gamma$ represents the \emph{information sparsity} of the task; a lower $\gamma$ implies that relevant information is highly concentrated in the first few dimensions.
\end{assumption}

\subsection{Derivation of Convergence Regimes}

\begin{theorem}[General Intervention Cost Bound]
\label{thm:general_bound}
Under Assumptions~\ref{ass:growth} and~\ref{ass:decay}, the expected intervention cost $E$ is bounded by a geometric series governed by the spectral ratio $\rho=r\gamma$:
\begin{equation}
E\le C k_1\sum_{i=1}^{L}(r\gamma)^{i-1}.
\label{eq:general_bound}
\end{equation}
\end{theorem}

\begin{proof}
Substituting the definitions of $k_i$ and $P(\ell^*=i)$ into the expectation:
\begin{equation}
\begin{aligned}
E &= \sum_{i=1}^{L} k_i\cdot P(\ell^*=i)
   \le \sum_{i=1}^{L}(k_1 r^{i-1})(C\gamma^{i-1}) \\
  &= C k_1\sum_{i=1}^{L}(r\gamma)^{i-1}.
\end{aligned}
\end{equation}
Let $\rho=r\gamma$. The summation $S_L=\sum_{j=0}^{L-1}\rho^j$ behaves differently depending on whether $\rho$ is below, equal to, or above $1$.
\end{proof}

\begin{remark}
The term $\rho=r\gamma$ represents the ratio between the \textbf{cost expansion rate} ($r$) and the \textbf{information gain rate} ($1/\gamma$). If $\rho<1$, information is gained faster than cost is added; if $\rho>1$, cost outpaces information.
\end{remark}

\begin{corollary}[Intervention Efficiency Regimes]
\label{cor:regimes}
The asymptotic scaling of $E$ in $K$ falls into one of three regimes:
\begin{enumerate}
    \item \textbf{Efficient regime} ($\gamma<1/r$): $E=\Theta(1)$.
    \item \textbf{Balanced regime} ($\gamma=1/r$): $E=\Theta(\log K)$.
    \item \textbf{Heavy-tailed regime} ($\gamma>1/r$): $E=\Theta(K^\alpha)$, where $\alpha=1+\log_r\gamma>0$.
\end{enumerate}
\end{corollary}

\begin{proof}
We analyze $S_L=\sum_{i=1}^{L}\rho^{i-1}$ from~\eqref{eq:general_bound} for each case.

\textbf{Case 1: $\rho<1$.} The geometric series converges:
\begin{equation}
S_L=\frac{1-\rho^L}{1-\rho}<\frac{1}{1-\rho},\qquad
E\le \frac{C k_1}{1-r\gamma}=\Theta(1).
\end{equation}

\textbf{Case 2: $\rho=1$.} Every term equals $1$, so $S_L=L$, and using Assumption~\ref{ass:growth}:
\begin{equation}
E\approx C k_1 L=C k_1\!\left(\log_r\frac{K}{k_1}+1\right)=\Theta(\log K).
\end{equation}

\textbf{Case 3: $\rho>1$.} The series diverges geometrically:
\begin{equation}
S_L=\frac{\rho^L-1}{\rho-1}\approx\frac{\rho^L}{\rho-1}=\Theta(\rho^L).
\end{equation}
Since $r^L\propto K$ and $\gamma^L=(K/k_1)^{\log_r\gamma}\propto K^{\log_r\gamma}$, we get
\begin{equation}
E\propto K\cdot K^{\log_r\gamma}=K^{1+\log_r\gamma}.
\end{equation}
With $\gamma>1/r$ we have $\log_r\gamma>-1$, so $\alpha=1+\log_r\gamma>0$.
\end{proof}

\begin{remark}
These regimes give a concrete target for Matryoshka Representation Learning:
\begin{itemize}
    \item \textbf{Regime 1 ($\Theta(1)$):} the ideal state---a \emph{critical mass} of discriminative information has been compressed into the first $k_1$ dimensions, and the system is fully scalable in $K$.
    \item \textbf{Regime 2 ($\Theta(\log K)$):} value of an extra concept just offsets its verification cost.
    \item \textbf{Regime 3 ($\Theta(K^\alpha)$):} information bottleneck failure; the nested structure offers little gain over a flat baseline.
\end{itemize}
\end{remark}

\section{Proof of Intervention Reduction}
\label{sec:proof_reduction}

\noindent We provide a formal proof that intervening on concepts in a Matryoshka CBM strictly reduces the upper bound on classification error, using the Hellman--Raviv inequality.

\subsection{Preliminaries and Definitions}

\noindent Let data be sampled from $P$ over $\mathcal{X}\times\mathcal{Y}\times\{0,1\}^K$, where $X\in\mathcal{X}$, $Y\in[C]$, and $C^*\in\{0,1\}^K$ is the ground-truth concept vector. The MCBM consists of:
\begin{enumerate}
\item A concept encoder $g:\mathcal{X}\to\mathbb{R}^K$ with $\hat{C}=\sigma(g(X))$.
\item A label predictor $f:\mathbb{R}^K\to\Delta^{C-1}$.
\end{enumerate}

\begin{definition}[Intervention Vector]
Let $\tilde{C}^{(k)}$ denote the concept vector after intervening on the first $k$ concepts under the Matryoshka ordering $\mathcal{M}$:
\begin{equation}
\tilde{C}^{(k)}=[c^*_1,\dots,c^*_k,\hat{c}_{k+1},\dots,\hat{c}_K].
\end{equation}
\end{definition}

\begin{definition}[Intervention Error]
\begin{equation}
P_e(k)=\Prob(f(\tilde{C}^{(k)})\neq Y)=\E\!\left[\Ind(f(\tilde{C}^{(k)})\neq Y)\right].
\end{equation}
\end{definition}

\subsection{Error Bound Derivation}

\begin{lemma}[Hellman--Raviv Inequality]
\label{lemma:hellman_raviv}
For any Bayesian classifier predicting $Y$ from observation $Z$, the error probability satisfies (in nats)
\begin{equation}
P_e\le \tfrac{1}{2}H(Y\mid Z).
\end{equation}
\end{lemma}

\begin{remark}
Reducing the conditional entropy $H(Y\mid Z)$ strictly lowers the ceiling on the error.
\end{remark}

\begin{theorem}[Upper Bound on Intervention Error]
\label{thm:error_bound}
Let $\epsilon=D_{KL}(P_{int}\,\|\,P_{train})$ be the intervention shift penalty. The probability of error after $k$ interventions satisfies
\begin{equation}
P_e(k)\le \tfrac{1}{2}\bigl(H(Y)-I(Y;\tilde{C}^{(k)})\bigr)+\sqrt{\tfrac{\epsilon}{2}}.
\label{eq:error_bound}
\end{equation}
\end{theorem}

\begin{proof}
\textbf{Step 1.} By Lemma~\ref{lemma:hellman_raviv},
\begin{equation}
P_e(k)\le \tfrac{1}{2} H(Y\mid \tilde{C}^{(k)}).
\end{equation}
\textbf{Step 2.} Using $I(Y;Z)=H(Y)-H(Y\mid Z)$,
\begin{equation}
H(Y\mid \tilde{C}^{(k)})=H(Y)-I(Y;\tilde{C}^{(k)}),
\end{equation}
yielding the entropic part of the bound.
\textbf{Step 3.} The classifier $f$ is trained under $P_{train}$ but evaluated under $P_{int}$. By Pinsker's inequality, the total-variation gap is bounded by
\begin{equation}
\delta(P_{int},P_{train})\le \sqrt{\tfrac{1}{2}D_{KL}(P_{int}\,\|\,P_{train})}=\sqrt{\tfrac{\epsilon}{2}},
\end{equation}
which, when added to the previous bound, gives~\eqref{eq:error_bound}.
\end{proof}

\begin{remark}
Equation~\eqref{eq:error_bound} mathematically justifies the Matryoshka objective: minimizing $P_e(k)$ requires maximizing $I(Y;\tilde{C}^{(k)})$ for \emph{every} prefix $k\in\mathcal{M}$, which standard CBM does only for $k=K$.
\end{remark}

\begin{corollary}[Monotonic Error Reduction]
\label{cor:monotonic}
Since $I(Y;\tilde{C}^{(k+1)})\ge I(Y;\tilde{C}^{(k)})$, the upper bound on the classification error is non-increasing in the intervention depth $k$.
\end{corollary}

\begin{remark}
This guarantees that any-time intervention is safe: stopping at depth $k$ never makes the theoretical bound worse than at depth $k-1$.
\end{remark}

\section{Supplementary Rebuttal Experiments}
\label{sec:app_rebuttal_experiments}

\subsection{CUB Baselines and Official Split Protocol}
\label{sec:app_cub_baselines}

Table~\ref{tab:app_cub_baselines} expands the CUB baseline comparison in the main paper. Sparse CBM is reproduced on the official CUB split because the original random split can leak test images into model selection. Label-free CBM obtains strong task accuracy, but its CLIP-aligned concepts are not guaranteed to match the human-annotated CUB concept vocabulary; therefore, they are less reliable for expert intervention than the fully grounded concepts used by MCBM.

\begin{table}[h]
\centering
\caption{\textbf{Comparison with CBM baselines on CUB under the official split.}}
\label{tab:app_cub_baselines}
\resizebox{0.82\linewidth}{!}{%
\scriptsize
\setlength{\tabcolsep}{4pt}
\begin{tabular}{lcl}
\toprule
\textbf{Method} & \textbf{Test Acc. (\%)} & \textbf{Remark} \\
\midrule
Sparse CBM & 53.75 & Reproduced on the official CUB split \\
Sequential CBM & 69.87 & Standard CBM baseline \\
Independent CBM & 70.38 & Standard CBM baseline \\
Label-free CBM~\cite{oikarinen2023labelfree} & 74.06 & Semantically ambiguous for intervention \\
MCBM (Standard, $K=112$) & 73.20 & Ours \\
Efficient MCBM ($K=112$) & 72.94 & Ours \\
\bottomrule
\end{tabular}%
}
\end{table}

\subsection{Concept Accuracy and Leakage Analysis}
\label{sec:app_leakage}

Jointly trained CBMs can leak task information through concepts that no longer behave as human-interpretable attributes~\cite{parisini2025leakage}. We check this in two ways. First, the $X\!\to\!C$ predictor remains accurate on the top mRMR prefixes: 92.44\% at $K=8$, 93.14\% at $K=16$, 93.80\% at $K=32$, and 96.27\% at $K=112$. Second, we train a strictly Sequential MCBM in which the concept extractor is trained before the Matryoshka task heads, removing gradient paths from the task loss into the concept representation. Even under this zero-leakage protocol, the model reaches 72.09\% task accuracy at $K=16$ and 73.75\% at $K=112$, showing that low-budget recovery is driven primarily by semantic compression through the mRMR order.

We also compare a flat joint CBM trained only with $K=16$ concepts against the matched $d=16$ Matryoshka head. The flat joint model reaches 65.84\% accuracy, whereas the matched Matryoshka head reaches 63.24\%. This gap is expected: the flat model has more freedom to distort its small concept set for task prediction, while nested supervision from larger heads regularizes MCBM toward a shared semantic ordering.

\subsection{Efficient MCBM Training Strategy}
\label{sec:app_efficient_training}

The default Efficient MCBM iterates over all nesting levels in each batch. To quantify the cost of this dense supervision, we compare it with a random-level variant that samples one granularity per batch. Random-level training is about $2.35\times$ faster, but Table~\ref{tab:app_random_level} shows that it damages the lowest budgets, where dense gradients are most important for enforcing the Matryoshka prior.

\begin{table}[h]
\centering
\caption{\textbf{Sensitivity of Efficient MCBM to training strategy on CUB.} Negative deltas indicate degradation from all-level training to random-level sampling.}
\label{tab:app_random_level}
\resizebox{0.9\linewidth}{!}{%
\scriptsize
\setlength{\tabcolsep}{4pt}
\begin{tabular}{c|ccc|ccc}
\toprule
\multirow{2}{*}{$K$} & \multicolumn{3}{c|}{\textbf{Accuracy (\%)}} & \multicolumn{3}{c}{\textbf{Macro F1 (\%)}} \\
 & All Levels & Random Level & $\Delta$ & All Levels & Random Level & $\Delta$ \\
\midrule
8   & 55.49 & 25.70 & -29.79 & 45.05 & 15.95 & -29.10 \\
16  & 88.57 & 70.69 & -17.88 & 84.85 & 62.95 & -21.90 \\
32  & 98.46 & 96.93 & -1.54  & 98.09 & 96.08 & -2.01 \\
64  & 99.50 & 99.50 & 0.00   & 99.34 & 99.34 & 0.00 \\
112 & 100.00 & 99.50 & -0.50 & 100.00 & 99.34 & -0.66 \\
\bottomrule
\end{tabular}%
}
\end{table}

\subsection{mRMR Stability and Weight-Based Ordering}
\label{sec:app_mrmr_stability}

To test whether the static hierarchy is brittle under redundant concepts, we rerun mRMR on CUB with five random seeds and compute the Intersection over Union of the top prefixes. The selected top-$K$ sets for $K=8,16,24,32$ have IoU $=1.00$ across all runs, indicating that the high-value prefix is stable even if long-tail concepts can permute.

We further compare mRMR against a task-weight ordering extracted from a fully trained independent $K=112$ CBM head. Ranking concepts by $\sum_c |W_{c,i}|$ yields Spearman correlation $0.934$ and Kendall's $\tau=0.792$ with mRMR ($p<10^{-30}$), and gives a near-identical intervention recovery curve. Thus, mRMR anticipates the same task-relevance hierarchy that a full predictor learns, without requiring a separate full-capacity CBM to be trained first.

\subsection{Matched-Head Versus Cross-Head Intervention}
\label{sec:app_cross_head}

One alternative is to intervene on the first $k$ concepts but still run the full $K=112$ classifier. This mixes $k$ hard expert labels with $K-k$ soft predicted logits, producing a distribution shift. Table~\ref{tab:app_cross_head} compares matched Matryoshka heads against the full head under this cross-head protocol.

\begin{table}[h]
\centering
\caption{\textbf{Cross-head intervention on CUB.} Matched heads are trained for the exact corrected prefix size, while the full head receives a mixture of corrected and uncorrected concepts.}
\label{tab:app_cross_head}
\resizebox{0.6\linewidth}{!}{%
\scriptsize
\setlength{\tabcolsep}{4pt}
\begin{tabular}{c|cc|c}
\toprule
\textbf{Corrected $k$} & \textbf{Matched Head Acc.} & \textbf{Full Head Acc.} & \textbf{Drop} \\
\midrule
16 & 88.71 & 83.03 & -5.68 \\
32 & 98.46 & 91.92 & -6.54 \\
\bottomrule
\end{tabular}%
}
\end{table}

\subsection{Empirical Geometric Decay}
\label{sec:app_geometric_decay}

For each initially misclassified CUB sample, we record the minimal sufficient intervention level $\ell^*$ at which the prediction becomes correct. Figure~\ref{fig:app_geometric_decay} plots the cumulative recovery curve, marginal gain, and fitted distribution of $\ell^*$. The fitted exponential trend has $\gamma=0.9636$ with $R^2=0.4808$. The moderate $R^2$ reflects local long-tail spikes from coupled fine-grained attributes, but the macro trend still concentrates mass on early interventions, supporting the geometric-decay assumption used in Theorem~\ref{thm:efficiency_regimes}.

\begin{figure}[h]
    \centering
    \includegraphics[width=0.72\linewidth]{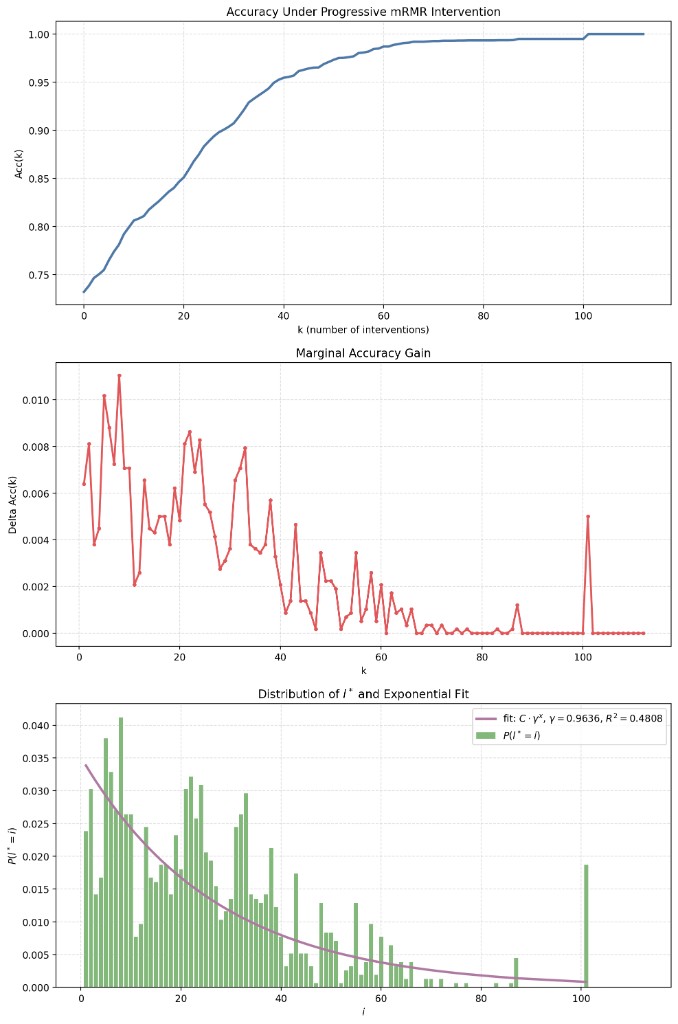}
    \caption{\textbf{Empirical geometric decay on CUB.} Accuracy recovers rapidly under progressive mRMR intervention; marginal gains decay overall, and the stopping-level distribution follows an exponential trend.}
    \label{fig:app_geometric_decay}
\end{figure}

\subsection{Causal CBM Comparison}
\label{sec:app_causal_comparison}

Causal and stochastic CBMs are complementary to MCBM. A causal graph can propagate a corrected concept to dependent concepts, improving inter-concept consistency. However, reading a graph backward from the task node identifies a Markov blanket, not an ordered subset optimized for a fixed budget. If multiple parents are redundant, verifying all of them wastes scarce expert attention. mRMR directly optimizes the prefix for maximum relevance and minimum redundancy.

The deployment distinction is also important. Dynamic graph-based methods generally keep the full concept vocabulary, graph structure, and propagation machinery active at inference time. They can mask or prioritize concepts logically, but the full model remains in memory. MCBM turns the ranking into an architectural prefix: when the budget is $k$, suffix concepts and unmatched heads can be removed, giving predictable physical truncation for edge or clinical deployments.

\section{Limitations}
\label{sec:limitations}

MCBM inherits several limitations from concept bottleneck modeling. First, its interpretability and intervention value depend on the quality and completeness of the concept vocabulary: missing, biased, or ambiguous concepts can limit both accuracy and human trust. Second, the theoretical efficiency guarantees rely on geometric information concentration in the ordered prefixes; our CUB analysis supports this trend empirically, but tasks with highly diffuse or weakly predictive concepts may fall into the heavier-tailed regimes described in Section~\ref{sec:theory}. Third, although we evaluate multiple datasets and backbones, the experiments focus on visual attribute benchmarks, so further validation is needed before deployment in clinical, scientific, or other high-stakes domains. Finally, physical truncation reduces memory and intervention cost, but it does not remove the need for domain-specific auditing, calibration, fairness evaluation, and expert oversight when the model is used in real decision pipelines.

\section{Broader Impact}
\label{sec:broader_impact}

MCBM can lower the cost of expert oversight by allowing practitioners to inspect a small, ordered prefix of human-understandable concepts before deciding whether further intervention is needed. This may improve accessibility of interpretable models in resource-constrained or high-stakes settings. At the same time, cheaper intervention does not remove the need for dataset and concept-set auditing: biased concept annotations, incomplete concept vocabularies, or over-reliance on imperfect concept predictions could still lead to unfair or unsafe decisions. We therefore view MCBM as a tool for reducing verification burden, not as a replacement for domain expert review, fairness evaluation, or deployment-specific risk assessment.

\section{Licenses and Existing Assets}
\label{sec:licenses_assets}

This paper and our original accompanying materials are released under the CC-BY 4.0 license. We use only public research assets cited in the main paper, including CUB-200-2011, LAD, CelebA, ImageNet-pretrained backbones, and prior CBM baselines, and we respect their original licenses and terms of use. Third-party datasets, pretrained weights, and baseline implementations remain governed by their original creators' licenses and are not redistributed as new assets.

\section{mRMR Concept Ranking for CUB Dataset}
\label{sec:cub_ranking}

This section presents the complete ranked list of 112 concepts for the CUB-200-2011 dataset, ordered by their mRMR score. The ranking is computed using the Minimum Redundancy Maximum Relevance algorithm as described in Section~\ref{sec:method}. Table~\ref{tab:cub_mrmr_ranking} details the Rank, Concept Name, mRMR Score ($I(c_i;y)-\text{Redundancy}$), Relevance ($I(c_i;y)$), and Redundancy ($\frac{1}{|\mathcal{S}|}\sum I(c_i;c_j)$) for each attribute.

\begin{longtable}{c l c c c}
\caption{\textbf{Full mRMR Concept Ranking for CUB Dataset.} Attributes are ordered by decreasing information density.}\label{tab:cub_mrmr_ranking}\\
\toprule
\textbf{Rank} & \textbf{Concept} & \textbf{Score} & \textbf{Relevance} & \textbf{Redundancy} \\
\midrule
\endfirsthead
\multicolumn{5}{c}{{\bfseries \tablename\ \thetable{} -- continued from previous page}} \\
\toprule
\textbf{Rank} & \textbf{Concept} & \textbf{Score} & \textbf{Relevance} & \textbf{Redundancy} \\
\midrule
\endhead
\midrule
\multicolumn{5}{r}Continued on next page \\
\bottomrule
\endfoot
\bottomrule
\endlastfoot

1 & has\_shape::perching-like & 0.6931 & 0.6931 & 0.0000 \\
2 & has\_bill\_color::orange & 0.6884 & 0.6930 & 0.0046 \\
3 & has\_underparts\_color::black & 0.6769 & 0.6830 & 0.0061 \\
4 & has\_upperparts\_color::orange & 0.6726 & 0.6760 & 0.0035 \\
5 & has\_belly\_color::buff & 0.6739 & 0.6823 & 0.0084 \\
6 & has\_bill\_shape::spatulate & 0.6615 & 0.6762 & 0.0146 \\
7 & has\_back\_pattern::multi-colored & 0.6580 & 0.6830 & 0.0249 \\
8 & has\_size::large\_(16\_-\_32\_in) & 0.6433 & 0.6577 & 0.0144 \\
9 & has\_under\_tail\_color::orange & 0.6330 & 0.6679 & 0.0349 \\
10 & has\_underparts\_color::buff & 0.6310 & 0.6493 & 0.0182 \\
11 & has\_shape::swallow-like & 0.6350 & 0.6733 & 0.0383 \\
12 & has\_belly\_color::black & 0.6307 & 0.6764 & 0.0457 \\
13 & has\_wing\_color::orange & 0.6274 & 0.6764 & 0.0490 \\
14 & has\_bill\_length::longer\_than\_head & 0.6189 & 0.6596 & 0.0407 \\
15 & has\_throat\_color::black & 0.6107 & 0.6580 & 0.0472 \\
16 & has\_forehead\_color::orange & 0.6078 & 0.6394 & 0.0316 \\
17 & has\_crown\_color::orange & 0.5853 & 0.6360 & 0.0508 \\
18 & has\_breast\_color::black & 0.5881 & 0.6584 & 0.0703 \\
19 & has\_eye\_color::buff & 0.5884 & 0.6420 & 0.0536 \\
20 & has\_upperparts\_color::rufous & 0.5796 & 0.5834 & 0.0037 \\
21 & has\_tail\_shape::rounded\_tail & 0.5785 & 0.6028 & 0.0242 \\
22 & has\_leg\_color::orange & 0.5786 & 0.5909 & 0.0123 \\
23 & has\_upper\_tail\_color::orange & 0.5660 & 0.6129 & 0.0469 \\
24 & has\_crown\_color::buff & 0.5534 & 0.5805 & 0.0271 \\
25 & has\_back\_color::rufous & 0.5510 & 0.5696 & 0.0185 \\
26 & has\_tail\_pattern::multi-colored & 0.5481 & 0.5654 & 0.0173 \\
27 & has\_wing\_color::blue & 0.5355 & 0.5578 & 0.0223 \\
28 & has\_primary\_color::orange & 0.5335 & 0.5906 & 0.0572 \\
29 & has\_bill\_shape::all-purpose & 0.5227 & 0.5556 & 0.0329 \\
30 & has\_back\_color::orange & 0.5232 & 0.5807 & 0.0576 \\
31 & has\_leg\_color::rufous & 0.5215 & 0.5275 & 0.0060 \\
32 & has\_wing\_pattern::striped & 0.5202 & 0.5292 & 0.0091 \\
33 & has\_wing\_color::rufous & 0.5184 & 0.5399 & 0.0215 \\
34 & has\_head\_pattern::eyering & 0.5093 & 0.5230 & 0.0137 \\
35 & has\_upperparts\_color::blue & 0.5071 & 0.5411 & 0.0340 \\
36 & has\_primary\_color::rufous & 0.4985 & 0.5228 & 0.0243 \\
37 & has\_back\_color::blue & 0.4958 & 0.5352 & 0.0394 \\
38 & has\_wing\_color::black & 0.4924 & 0.5056 & 0.0132 \\
39 & has\_nape\_color::orange & 0.4866 & 0.5351 & 0.0485 \\
40 & has\_wing\_pattern::spotted & 0.4774 & 0.5009 & 0.0235 \\
41 & has\_primary\_color::black & 0.4740 & 0.5052 & 0.0313 \\
42 & has\_throat\_color::orange & 0.4635 & 0.4983 & 0.0349 \\
43 & has\_upper\_tail\_color::rufous & 0.4568 & 0.4740 & 0.0172 \\
44 & has\_primary\_color::blue & 0.4489 & 0.4866 & 0.0377 \\
45 & has\_nape\_color::black & 0.4487 & 0.4785 & 0.0298 \\
46 & has\_size::very\_large\_(32\_-\_72\_in) & 0.4433 & 0.4764 & 0.0331 \\
47 & has\_upperparts\_color::black & 0.4395 & 0.4641 & 0.0246 \\
48 & has\_nape\_color::rufous & 0.4341 & 0.4491 & 0.0150 \\
49 & has\_wing\_shape::rounded-wings & 0.4193 & 0.4236 & 0.0043 \\
50 & has\_upper\_tail\_color::blue & 0.4192 & 0.4572 & 0.0380 \\
51 & has\_breast\_color::orange & 0.4139 & 0.4441 & 0.0302 \\
52 & has\_under\_tail\_color::black & 0.4109 & 0.4373 & 0.0263 \\
53 & has\_under\_tail\_color::rufous & 0.4075 & 0.4244 & 0.0169 \\
54 & has\_belly\_color::grey & 0.4024 & 0.4173 & 0.0150 \\
55 & has\_underparts\_color::red & 0.4004 & 0.4156 & 0.0151 \\
56 & has\_underparts\_color::orange & 0.3960 & 0.4303 & 0.0342 \\
57 & has\_breast\_color::grey & 0.3956 & 0.4166 & 0.0210 \\
58 & has\_under\_tail\_color::blue & 0.3924 & 0.4250 & 0.0326 \\
59 & has\_underparts\_color::grey & 0.3827 & 0.4083 & 0.0256 \\
60 & has\_crown\_color::blue & 0.3810 & 0.4150 & 0.0339 \\
61 & has\_tail\_pattern::striped & 0.3792 & 0.3887 & 0.0095 \\
62 & has\_bill\_color::rufous & 0.3714 & 0.3800 & 0.0086 \\
63 & has\_forehead\_color::blue & 0.3675 & 0.4058 & 0.0383 \\
64 & has\_wing\_color::red & 0.3607 & 0.3878 & 0.0270 \\
65 & has\_nape\_color::blue & 0.3567 & 0.3853 & 0.0286 \\
66 & has\_upperparts\_color::red & 0.3535 & 0.3783 & 0.0248 \\
67 & has\_primary\_color::grey & 0.3527 & 0.3798 & 0.0271 \\
68 & has\_breast\_color::red & 0.3509 & 0.3700 & 0.0191 \\
69 & has\_forehead\_color::rufous & 0.3469 & 0.3586 & 0.0117 \\
70 & has\_back\_pattern::spotted & 0.3460 & 0.3685 & 0.0225 \\
71 & has\_upper\_tail\_color::black & 0.3443 & 0.3658 & 0.0216 \\
72 & has\_belly\_color::orange & 0.3441 & 0.3752 & 0.0311 \\
73 & has\_crown\_color::rufous & 0.3440 & 0.3586 & 0.0147 \\
74 & has\_back\_color::black & 0.3412 & 0.3658 & 0.0246 \\
75 & has\_belly\_color::red & 0.3381 & 0.3585 & 0.0203 \\
76 & has\_leg\_color::red & 0.3308 & 0.3459 & 0.0150 \\
77 & has\_back\_color::red & 0.3219 & 0.3476 & 0.0258 \\
78 & has\_throat\_color::grey & 0.3194 & 0.3382 & 0.0188 \\
79 & has\_underparts\_color::rufous & 0.3122 & 0.3199 & 0.0076 \\
80 & has\_size::medium\_(9\_-\_16\_in) & 0.3106 & 0.3168 & 0.0062 \\
81 & has\_primary\_color::red & 0.2896 & 0.3053 & 0.0157 \\
82 & has\_breast\_color::rufous & 0.2882 & 0.2964 & 0.0082 \\
83 & has\_bill\_shape::curved\_(up\_or\_down) & 0.2797 & 0.2920 & 0.0123 \\
84 & has\_breast\_pattern::striped & 0.2771 & 0.2817 & 0.0045 \\
85 & has\_belly\_color::rufous & 0.2757 & 0.2866 & 0.0109 \\
86 & has\_nape\_color::red & 0.2759 & 0.2914 & 0.0155 \\
87 & has\_under\_tail\_color::red & 0.2701 & 0.2904 & 0.0203 \\
88 & has\_upperparts\_color::grey & 0.2695 & 0.2817 & 0.0122 \\
89 & has\_forehead\_color::black & 0.2682 & 0.2888 & 0.0206 \\
90 & has\_upper\_tail\_color::red & 0.2677 & 0.2904 & 0.0227 \\
91 & has\_forehead\_color::grey & 0.2569 & 0.2683 & 0.0113 \\
92 & has\_breast\_pattern::spotted & 0.2487 & 0.2632 & 0.0144 \\
93 & has\_back\_pattern::striped & 0.2466 & 0.2536 & 0.0070 \\
94 & has\_crown\_color::black & 0.2460 & 0.2638 & 0.0178 \\
95 & has\_tail\_pattern::spotted & 0.2328 & 0.2429 & 0.0101 \\
96 & has\_breast\_color::blue & 0.2261 & 0.2423 & 0.0162 \\
97 & has\_throat\_color::red & 0.2214 & 0.2400 & 0.0186 \\
98 & has\_wing\_color::grey & 0.2179 & 0.2292 & 0.0113 \\
99 & has\_nape\_color::grey & 0.2142 & 0.2285 & 0.0144 \\
100 & has\_underparts\_color::blue & 0.2135 & 0.2285 & 0.0150 \\
101 & has\_bill\_shape::needle & 0.2133 & 0.2248 & 0.0115 \\
102 & has\_bill\_color::red & 0.2032 & 0.2124 & 0.0092 \\
103 & has\_crown\_color::grey & 0.2018 & 0.2141 & 0.0123 \\
104 & has\_belly\_color::blue & 0.1975 & 0.2117 & 0.0142 \\
105 & has\_bill\_length::shorter\_than\_head & 0.1938 & 0.1994 & 0.0056 \\
106 & has\_shape::long-legged-like & 0.1920 & 0.1977 & 0.0057 \\
107 & has\_bill\_color::buff & 0.1920 & 0.1994 & 0.0074 \\
108 & has\_head\_pattern::unique\_pattern & 0.1917 & 0.1984 & 0.0067 \\
109 & has\_throat\_color::rufous & 0.1910 & 0.1977 & 0.0066 \\
110 & has\_wing\_pattern::solid & 0.1896 & 0.1967 & 0.0071 \\
111 & has\_back\_color::grey & 0.1862 & 0.2009 & 0.0147 \\
112 & has\_eye\_color::orange & 0.1523 & 0.1588 & 0.0065 \\
\end{longtable}

\section{mRMR Concept Ranking for CelebA Dataset}
\label{sec:celeba_ranking}

This section presents the complete ranked list of 40 binary attributes for the CelebA dataset. As shown in the experiments, the top-ranked concepts (e.g., ``Attractive", ``Young", ``Heavy\_Makeup") carry the vast majority of the discriminative information for identity-related tasks, explaining the rapid performance saturation observed at $K=8$.

\begin{longtable}{c l c c c}
\caption{\textbf{Full mRMR Concept Ranking for CelebA Dataset.} Attributes are ordered by decreasing information density.}\\
\label{tab:celeba_mrmr_ranking} \\
\toprule
\textbf{Rank} & \textbf{Concept} & \textbf{Score} & \textbf{Relevance} & \textbf{Redundancy} \\
\midrule
\endfirsthead

\multicolumn{5}{c}%
{{\bfseries \tablename\ \thetable{} -- continued from previous page}} \\
\toprule
\textbf{Rank} & \textbf{Concept} & \textbf{Score} & \textbf{Relevance} & \textbf{Redundancy} \\
\midrule
\endhead

\midrule
\multicolumn{5}{r}{{Continued on next page}} \\
\bottomrule
\endfoot

\bottomrule
\endlastfoot

1 & Attractive & 0.6928 & 0.6928 & 0.0000 \\
2 & 5\_o\_Clock\_Shadow & 0.0000 & 0.0020 & 0.0020 \\
3 & Young & 0.0403 & 0.0808 & 0.0404 \\
4 & Heavy\_Makeup & 0.0501 & 0.1205 & 0.0704 \\
5 & Chubby & 0.0128 & 0.0344 & 0.0217 \\
6 & Blurry & 0.0117 & 0.0182 & 0.0065 \\
7 & Eyeglasses & 0.0142 & 0.0283 & 0.0141 \\
8 & Wearing\_Lipstick & 0.0246 & 0.1204 & 0.0958 \\
9 & Pointy\_Nose & 0.0121 & 0.0266 & 0.0145 \\
10 & Big\_Nose & 0.0126 & 0.0395 & 0.0269 \\
11 & Oval\_Face & 0.0114 & 0.0191 & 0.0077 \\
12 & Gray\_Hair & 0.0115 & 0.0252 & 0.0137 \\
13 & Receding\_Hairline & 0.0086 & 0.0168 & 0.0082 \\
14 & Double\_Chin & 0.0086 & 0.0266 & 0.0179 \\
15 & Arched\_Eyebrows & 0.0095 & 0.0323 & 0.0228 \\
16 & Wavy\_Hair & 0.0089 & 0.0235 & 0.0146 \\
17 & Wearing\_Hat & 0.0059 & 0.0102 & 0.0043 \\
18 & Bald & 0.0062 & 0.0132 & 0.0070 \\
19 & Male & 0.0090 & 0.0800 & 0.0710 \\
20 & Smiling & 0.0057 & 0.0110 & 0.0053 \\
21 & Brown\_Hair & 0.0041 & 0.0088 & 0.0048 \\
22 & Rosy\_Cheeks & 0.0032 & 0.0146 & 0.0113 \\
23 & Mustache & 0.0029 & 0.0106 & 0.0077 \\
24 & Pale\_Skin & 0.0028 & 0.0038 & 0.0010 \\
25 & Narrow\_Eyes & 0.0018 & 0.0027 & 0.0009 \\
26 & Bags\_Under\_Eyes & 0.0020 & 0.0161 & 0.0141 \\
27 & Blond\_Hair & 0.0020 & 0.0123 & 0.0103 \\
28 & Goatee & 0.0009 & 0.0113 & 0.0103 \\
29 & Wearing\_Necktie & 0.0005 & 0.0128 & 0.0123 \\
30 & Bangs & -0.0019 & 0.0018 & 0.0037 \\
31 & Big\_Lips & -0.0020 & 0.0020 & 0.0040 \\
32 & Straight\_Hair & -0.0025 & 0.0009 & 0.0033 \\
33 & Bushy\_Eyebrows & -0.0035 & 0.0009 & 0.0044 \\
34 & Sideburns & -0.0040 & 0.0051 & 0.0091 \\
35 & Wearing\_Necklace & -0.0041 & 0.0024 & 0.0065 \\
36 & High\_Cheekbones & -0.0041 & 0.0112 & 0.0153 \\
37 & Wearing\_Earrings & -0.0049 & 0.0078 & 0.0127 \\
38 & Black\_Hair & -0.0051 & 0.0000 & 0.0051 \\
39 & No\_Beard & -0.0076 & 0.0200 & 0.0276 \\
40 & Mouth\_Slightly\_Open & -0.0080 & 0.0002 & 0.0082 \\
\end{longtable}

\section{mRMR Concept Ranking for LAD Dataset}
\label{sec:lad_ranking}

This section presents the complete ranked list of 359 concepts for the Large-scale Attribute Dataset (LAD), ordered by their mRMR score. The hierarchy reflects the ``Threshold Effect" observed in our experiments, where early concepts capture broad, structural attributes (e.g., ``material", ``state"), while later concepts refine specific details.

\begin{longtable}{c l c c c}
\caption{\textbf{Full mRMR Concept Ranking for LAD Dataset.} Attributes are ordered by decreasing information density.}\\
\label{tab:lad_mrmr_ranking} \\
\toprule
\textbf{Rank} & \textbf{Concept} & \textbf{Score} & \textbf{Relevance} & \textbf{Redundancy} \\
\midrule
\endfirsthead

\multicolumn{5}{c}%
{{\bfseries \tablename\ \thetable{} -- continued from previous page}} \\
\toprule
\textbf{Rank} & \textbf{Concept} & \textbf{Score} & \textbf{Relevance} & \textbf{Redundancy} \\
\midrule
\endhead

\midrule
\multicolumn{5}{r}{{Continued on next page}} \\
\bottomrule
\endfoot

\bottomrule
\endlastfoot

1 & current state: is complete & 0.5517 & 0.5517 & 0.0000 \\
2 & material: is made of metal & 0.4689 & 0.5473 & 0.0784 \\
3 & material: is made of metal & 0.4538 & 0.5228 & 0.0691 \\
4 & other body parts: has eyes & 0.3893 & 0.4355 & 0.0462 \\
5 & taste: tastes sweet & 0.3802 & 0.5449 & 0.1647 \\
6 & aim: is for family & 0.3860 & 0.5410 & 0.1550 \\
7 & price: is expensive & 0.3871 & 0.5195 & 0.1323 \\
8 & edibility: is common & 0.3719 & 0.5383 & 0.1664 \\
9 & aim: is for display & 0.3674 & 0.5225 & 0.1551 \\
10 & function: can be driven & 0.3632 & 0.5174 & 0.1542 \\
11 & growth: grows on trees & 0.3683 & 0.5070 & 0.1386 \\
12 & safety: is safe & 0.3662 & 0.5218 & 0.1556 \\
13 & edibility: has a high water content & 0.3631 & 0.5294 & 0.1663 \\
14 & material: is made of plastic & 0.3617 & 0.5003 & 0.1386 \\
15 & aim: is for business & 0.3601 & 0.5422 & 0.1821 \\
16 & epidermis: has peel & 0.3558 & 0.4970 & 0.1411 \\
17 & function: can move & 0.3577 & 0.5181 & 0.1604 \\
18 & aim: is for personal & 0.3552 & 0.5410 & 0.1858 \\
19 & current state: is raw & 0.3581 & 0.5458 & 0.1876 \\
20 & other body parts: has a nose & 0.3530 & 0.4088 & 0.0558 \\
21 & parts: has seats & 0.3570 & 0.5075 & 0.1505 \\
22 & material: is made of plastic & 0.3579 & 0.5473 & 0.1894 \\
23 & habit: is active & 0.3443 & 0.4073 & 0.0630 \\
24 & existence: is common & 0.3485 & 0.5468 & 0.1983 \\
25 & parts: has a engine & 0.3518 & 0.4923 & 0.1405 \\
26 & hardness: is soft & 0.3523 & 0.4600 & 0.1077 \\
27 & epidermis: is smooth & 0.3473 & 0.4566 & 0.1093 \\
28 & shape: is long & 0.3486 & 0.4837 & 0.1350 \\
29 & existence: is movable & 0.3501 & 0.4803 & 0.1302 \\
30 & size: is small (compared to apples) & 0.3474 & 0.4486 & 0.1011 \\
31 & parts: has lights & 0.3483 & 0.4899 & 0.1416 \\
32 & edibility: can be eaten directly & 0.3453 & 0.4470 & 0.1017 \\
33 & aim: is for office use & 0.3479 & 0.4698 & 0.1219 \\
34 & function: can carry a small number (\u226410) of passengers & 0.3422 & 0.4479 & 0.1057 \\
35 & size: is big (compared to a mobile phone) & 0.3424 & 0.4437 & 0.1013 \\
36 & fitness: fits people with a pointed nose & 0.3422 & 0.3721 & 0.0299 \\
37 & other body parts: has a backbone & 0.3415 & 0.4009 & 0.0594 \\
38 & aim: is for civil use & 0.3424 & 0.4374 & 0.0950 \\
39 & power: is a low-power (\u2264100 w) device & 0.3412 & 0.4424 & 0.1012 \\
40 & weight: weighs tons & 0.3421 & 0.4565 & 0.1144 \\
41 & appearance: has soft skin & 0.3382 & 0.3911 & 0.0529 \\
42 & speed: moves fast & 0.3404 & 0.4613 & 0.1209 \\
43 & sound: is quiet & 0.3394 & 0.4238 & 0.0844 \\
44 & edibility: has seeds & 0.3401 & 0.4066 & 0.0665 \\
45 & shape: is globular & 0.3393 & 0.4161 & 0.0769 \\
46 & parts: has indicator lights & 0.3377 & 0.4217 & 0.0839 \\
47 & other body parts: has a tongue & 0.3326 & 0.3931 & 0.0605 \\
48 & parts: has doors & 0.3343 & 0.4382 & 0.1040 \\
49 & taste: tastes sour & 0.3307 & 0.4060 & 0.0753 \\
50 & parts: has a horn & 0.3321 & 0.4297 & 0.0976 \\
51 & parts: has keys & 0.3325 & 0.4137 & 0.0812 \\
52 & fitness: fits fair-skinned people & 0.3311 & 0.3638 & 0.0327 \\
53 & habit: moves fast & 0.3315 & 0.3868 & 0.0553 \\
54 & color: is black & 0.3320 & 0.4191 & 0.0871 \\
55 & parts: has windows & 0.3332 & 0.4355 & 0.1023 \\
56 & edibility: needs to be skinned & 0.3295 & 0.3813 & 0.0519 \\
57 & habit: has muscle & 0.3286 & 0.3927 & 0.0641 \\
58 & size: is big (compared to cars) & 0.3292 & 0.4146 & 0.0854 \\
59 & edibility: has nutlets & 0.3251 & 0.3809 & 0.0557 \\
60 & parts: has a brake & 0.3248 & 0.4081 & 0.0832 \\
61 & feeling: is lively & 0.3216 & 0.3549 & 0.0333 \\
62 & other body parts: has ears & 0.3206 & 0.3849 & 0.0643 \\
63 & parts: has a plug & 0.3207 & 0.3791 & 0.0584 \\
64 & habit: is friendly & 0.3189 & 0.3685 & 0.0497 \\
65 & usage scenarios: can be used on urban roads & 0.3190 & 0.3963 & 0.0772 \\
66 & function: can give out lights & 0.3189 & 0.3835 & 0.0646 \\
67 & safety: is safe & 0.3166 & 0.3929 & 0.0763 \\
68 & habitat: lives on the ground & 0.3149 & 0.3655 & 0.0507 \\
69 & fitness: fits people with small eyes & 0.3150 & 0.3502 & 0.0352 \\
70 & habit: lives in groups & 0.3128 & 0.3566 & 0.0438 \\
71 & habit: is quiet & 0.3104 & 0.3533 & 0.0429 \\
72 & outside color: is yellow & 0.3105 & 0.3510 & 0.0406 \\
73 & size: is small (compared to pigs) & 0.3095 & 0.3543 & 0.0448 \\
74 & habit: is warm-blooded & 0.3049 & 0.3696 & 0.0647 \\
75 & growth: is tropical & 0.3038 & 0.3423 & 0.0385 \\
76 & habit: is weak & 0.3003 & 0.3472 & 0.0469 \\
77 & diet: eats meat & 0.2991 & 0.3514 & 0.0524 \\
78 & weight: weighs kilograms & 0.2994 & 0.3501 & 0.0507 \\
79 & parts: has a steering wheel & 0.2999 & 0.3659 & 0.0660 \\
80 & function: can carry a large quantity ($>1$ tons) of goods & 0.2955 & 0.3505 & 0.0550 \\
81 & medicinal property: is cool or cold & 0.2949 & 0.3303 & 0.0354 \\
82 & hardness: is hard & 0.2944 & 0.3240 & 0.0297 \\
83 & usage scenarios: can be used on rural roads & 0.2932 & 0.3548 & 0.0616 \\
84 & fitness: fits people with ear rings & 0.2922 & 0.3212 & 0.0290 \\
85 & neck: has a long neck & 0.2900 & 0.3466 & 0.0566 \\
86 & power: consumes diesel oil & 0.2884 & 0.3482 & 0.0598 \\
87 & behaviour: can walk & 0.2869 & 0.3400 & 0.0531 \\
88 & smell: is fragrant & 0.2860 & 0.3200 & 0.0340 \\
89 & fitness: fits people with heavy make-up & 0.2861 & 0.3170 & 0.0308 \\
90 & diet: eats plants & 0.2865 & 0.3348 & 0.0484 \\
91 & color: is gray & 0.2832 & 0.3217 & 0.0385 \\
92 & habit: is timid & 0.2822 & 0.3230 & 0.0407 \\
93 & paws: has paws & 0.2815 & 0.3395 & 0.0581 \\
94 & outside color: is red & 0.2809 & 0.3128 & 0.0319 \\
95 & medicinal property: is warm or hot & 0.2766 & 0.3069 & 0.0303 \\
96 & shape: is long & 0.2752 & 0.3020 & 0.0268 \\
97 & teeth: has teeth & 0.2751 & 0.3235 & 0.0484 \\
98 & diet: eats leaves & 0.2723 & 0.3210 & 0.0487 \\
99 & parts: has a number plate & 0.2725 & 0.3257 & 0.0532 \\
100 & feeling: is sexy & 0.2690 & 0.2951 & 0.0261 \\
101 & epidermis: is rough & 0.2647 & 0.2886 & 0.0239 \\
102 & aim: is for entertainment & 0.2637 & 0.2963 & 0.0326 \\
103 & behaviour: can swim & 0.2634 & 0.3082 & 0.0448 \\
104 & safety: is dangerous & 0.2633 & 0.2905 & 0.0271 \\
105 & feeling: is pure & 0.2631 & 0.2870 & 0.0239 \\
106 & function: can carry a small quantity (\u2264 1 ton) of goods & 0.2612 & 0.2927 & 0.0315 \\
107 & shape: is ellipsoidal & 0.2600 & 0.2837 & 0.0237 \\
108 & habit: can milk & 0.2590 & 0.3070 & 0.0480 \\
109 & feeling: is elegant & 0.2583 & 0.2848 & 0.0265 \\
110 & diet: eats seeds & 0.2567 & 0.3023 & 0.0455 \\
111 & parts: has four wheels & 0.2570 & 0.3004 & 0.0434 \\
112 & material: is made of glass & 0.2569 & 0.2844 & 0.0275 \\
113 & parts: has a motor & 0.2566 & 0.2855 & 0.0289 \\
114 & limb: has four legs & 0.2559 & 0.3055 & 0.0496 \\
115 & feeling: is simple & 0.2557 & 0.2800 & 0.0243 \\
116 & appearance: is smooth & 0.2547 & 0.2869 & 0.0321 \\
117 & feeling: is cute & 0.2453 & 0.2703 & 0.0249 \\
118 & limb: has short legs & 0.2427 & 0.2803 & 0.0376 \\
119 & shape: is flat & 0.2418 & 0.2648 & 0.0230 \\
120 & color: is black & 0.2416 & 0.2691 & 0.0276 \\
121 & tail: has a long tail & 0.2397 & 0.2726 & 0.0328 \\
122 & aim: is for family & 0.2384 & 0.2685 & 0.0301 \\
123 & parts: has an audio & 0.2378 & 0.2608 & 0.0229 \\
124 & existence: is fixed & 0.2351 & 0.2539 & 0.0189 \\
125 & sound: is noisy & 0.2330 & 0.2567 & 0.0237 \\
126 & diet: eats nectar & 0.2277 & 0.2575 & 0.0298 \\
127 & diet: eats insects & 0.2250 & 0.2564 & 0.0314 \\
128 & color: is white & 0.2250 & 0.2504 & 0.0253 \\
129 & color: is gray & 0.2238 & 0.2568 & 0.0330 \\
130 & usage mode: is handheld & 0.2235 & 0.2412 & 0.0177 \\
131 & function: can transmit signal & 0.2227 & 0.2447 & 0.0220 \\
132 & behaviour: can jump & 0.2219 & 0.2573 & 0.0354 \\
133 & color: is gray & 0.2219 & 0.2429 & 0.0210 \\
134 & parts: has a fan & 0.2207 & 0.2442 & 0.0234 \\
135 & outside color: is cyan & 0.2201 & 0.2360 & 0.0159 \\
136 & aim: is for engineering & 0.2193 & 0.2436 & 0.0243 \\
137 & appearance: is furry & 0.2178 & 0.2537 & 0.0360 \\
138 & feeling: is easy to be messy & 0.2164 & 0.2344 & 0.0180 \\
139 & function: can give out sound & 0.2164 & 0.2362 & 0.0198 \\
140 & habit: is strong & 0.2161 & 0.2457 & 0.0296 \\
141 & current state: has been processed & 0.2131 & 0.2295 & 0.0163 \\
142 & color: is white & 0.2098 & 0.2268 & 0.0170 \\
143 & habit: is smart & 0.2080 & 0.2375 & 0.0294 \\
144 & medicinal property: is mild & 0.2075 & 0.2204 & 0.0129 \\
145 & parts: has a battery & 0.2059 & 0.2226 & 0.0167 \\
146 & paws: has pads & 0.2054 & 0.2400 & 0.0346 \\
147 & size: is big (compared to apples) & 0.2045 & 0.2187 & 0.0142 \\
148 & tail: has a short tail & 0.2043 & 0.2308 & 0.0265 \\
149 & limb: has long legs & 0.2039 & 0.2295 & 0.0256 \\
150 & aim: is for military & 0.2014 & 0.2174 & 0.0161 \\
151 & function: can show characters & 0.2002 & 0.2194 & 0.0192 \\
152 & color: is black & 0.1982 & 0.2183 & 0.0200 \\
153 & aim: is for safety & 0.1981 & 0.2146 & 0.0164 \\
154 & function: can carry a larger number ($>10$)  of passengers & 0.1978 & 0.2142 & 0.0164 \\
155 & power: consumes special fuel & 0.1973 & 0.2126 & 0.0154 \\
156 & appearance: is hairless & 0.1965 & 0.2085 & 0.0120 \\
157 & habit: is the predator & 0.1952 & 0.2133 & 0.0181 \\
158 & habit: is noisy & 0.1945 & 0.2170 & 0.0225 \\
159 & size: is small (compared to a mobile phone) & 0.1904 & 0.2022 & 0.0118 \\
160 & aim: is for communication & 0.1899 & 0.2065 & 0.0166 \\
161 & diet: eats fish & 0.1897 & 0.2085 & 0.0188 \\
162 & marking: has patches & 0.1891 & 0.2127 & 0.0236 \\
163 & power: is a high-power (>100 w) device & 0.1888 & 0.2039 & 0.0152 \\
164 & current state: is shredded & 0.1834 & 0.1971 & 0.0136 \\
165 & shape: is cubic & 0.1830 & 0.1984 & 0.0153 \\
166 & appearance: has whiskers & 0.1831 & 0.2084 & 0.0254 \\
167 & color: is golden & 0.1790 & 0.1927 & 0.0137 \\
168 & habit: lives in solitary & 0.1779 & 0.1958 & 0.0178 \\
169 & aim: is for rescue & 0.1752 & 0.1912 & 0.0160 \\
170 & speed: moves slowly & 0.1750 & 0.1849 & 0.0099 \\
171 & function: can storage data & 0.1746 & 0.1872 & 0.0126 \\
172 & fitness: fits people with a hat & 0.1747 & 0.1868 & 0.0122 \\
173 & size: is big (compared to pigs) & 0.1731 & 0.1921 & 0.0190 \\
174 & function: can fly & 0.1728 & 0.1847 & 0.0119 \\
175 & wing: has long wings & 0.1727 & 0.1886 & 0.0159 \\
176 & usage scenarios: can be used in the sky & 0.1718 & 0.1847 & 0.0128 \\
177 & behaviour: can climb trees & 0.1716 & 0.1917 & 0.0201 \\
178 & habit: is cold-blooded & 0.1705 & 0.1804 & 0.0099 \\
179 & parts: has only one usb interface & 0.1693 & 0.1793 & 0.0100 \\
180 & habitat: is endangered & 0.1670 & 0.1845 & 0.0176 \\
181 & power: consumes gasoline & 0.1660 & 0.1800 & 0.0140 \\
182 & parts: has a propeller & 0.1659 & 0.1776 & 0.0117 \\
183 & teeth: has buck teeth & 0.1620 & 0.1844 & 0.0224 \\
184 & outside color: is green & 0.1616 & 0.1703 & 0.0087 \\
185 & size: is middle (compared to apples) & 0.1609 & 0.1704 & 0.0095 \\
186 & epidermis: has crust & 0.1596 & 0.1671 & 0.0076 \\
187 & feeling: is gorgeous & 0.1588 & 0.1669 & 0.0081 \\
188 & behaviour: can fly & 0.1577 & 0.1713 & 0.0136 \\
189 & habitat: is domestic & 0.1570 & 0.1726 & 0.0157 \\
190 & habit: moves slow & 0.1562 & 0.1664 & 0.0101 \\
191 & parts: has volume control buttons & 0.1556 & 0.1670 & 0.0114 \\
192 & growth: grows on the ground & 0.1550 & 0.1630 & 0.0080 \\
193 & parts: has three wheels & 0.1549 & 0.1659 & 0.0110 \\
194 & color: is brown & 0.1542 & 0.1713 & 0.0171 \\
195 & marking: has spots & 0.1540 & 0.1620 & 0.0080 \\
196 & parts: has a small screen (\u226410 inches) & 0.1533 & 0.1640 & 0.0107 \\
197 & size: is middle (compared to cars) & 0.1528 & 0.1650 & 0.0122 \\
198 & appearance: has tough skin & 0.1489 & 0.1590 & 0.0101 \\
199 & outside color: is white & 0.1489 & 0.1569 & 0.0080 \\
200 & parts: has armors & 0.1472 & 0.1578 & 0.0106 \\
201 & parts: has a headphone jack & 0.1451 & 0.1551 & 0.0100 \\
202 & feeling: is cool & 0.1442 & 0.1509 & 0.0068 \\
203 & parts: has a cable & 0.1418 & 0.1497 & 0.0079 \\
204 & function: can communicate in a short distance & 0.1417 & 0.1525 & 0.0107 \\
205 & marking: has stripes & 0.1403 & 0.1515 & 0.0112 \\
206 & feeling: is mature & 0.1397 & 0.1461 & 0.0064 \\
207 & habitat: lives in the plains & 0.1395 & 0.1529 & 0.0134 \\
208 & function: can play music & 0.1381 & 0.1482 & 0.0101 \\
209 & habit: can hibernate & 0.1375 & 0.1464 & 0.0089 \\
210 & function: can play videos & 0.1371 & 0.1473 & 0.0102 \\
211 & appearance: has feathers & 0.1346 & 0.1476 & 0.0131 \\
212 & limb: has two legs & 0.1339 & 0.1476 & 0.0137 \\
213 & beak: has a short beak & 0.1333 & 0.1476 & 0.0143 \\
214 & limb: is bipedal & 0.1326 & 0.1476 & 0.0150 \\
215 & habit: is fierce & 0.1326 & 0.1446 & 0.0120 \\
216 & parts: has more than four wheels & 0.1321 & 0.1407 & 0.0086 \\
217 & epidermis: is hairy & 0.1301 & 0.1360 & 0.0059 \\
218 & limb: has two arms & 0.1289 & 0.1413 & 0.0124 \\
219 & behaviour: can lay eggs & 0.1285 & 0.1409 & 0.0123 \\
220 & function: can float & 0.1282 & 0.1359 & 0.0077 \\
221 & function: can heat & 0.1277 & 0.1349 & 0.0072 \\
222 & usage scenarios: can be used in the sea & 0.1277 & 0.1359 & 0.0082 \\
223 & habitat: lives in mountains & 0.1271 & 0.1364 & 0.0093 \\
224 & habitat: lives in the forest & 0.1264 & 0.1363 & 0.0099 \\
225 & parts: has weapons & 0.1254 & 0.1335 & 0.0082 \\
226 & taste: tastes bitter & 0.1211 & 0.1268 & 0.0057 \\
227 & habitat: lives in water & 0.1211 & 0.1267 & 0.0055 \\
228 & habit: lives in nests & 0.1206 & 0.1320 & 0.0114 \\
229 & outside color: is orange & 0.1205 & 0.1259 & 0.0054 \\
230 & outside color: is violet & 0.1193 & 0.1245 & 0.0051 \\
231 & fitness: fits people with moustache & 0.1193 & 0.1235 & 0.0041 \\
232 & size: is small (compared to cars) & 0.1185 & 0.1238 & 0.0053 \\
233 & habitat: lives in the jungle & 0.1183 & 0.1280 & 0.0097 \\
234 & beak: has a needle beak & 0.1176 & 0.1308 & 0.0132 \\
235 & parts: has only one network interface & 0.1161 & 0.1233 & 0.0072 \\
236 & diet: eats plankton & 0.1144 & 0.1200 & 0.0055 \\
237 & color: is yellow & 0.1140 & 0.1242 & 0.0101 \\
238 & color: is brown & 0.1119 & 0.1166 & 0.0048 \\
239 & usage scenarios: can be used in the river & 0.1114 & 0.1182 & 0.0068 \\
240 & behaviour: can fish & 0.1105 & 0.1175 & 0.0070 \\
241 & fitness: fits people with beard & 0.1095 & 0.1139 & 0.0044 \\
242 & power: consumes electricity & 0.1092 & 0.1137 & 0.0045 \\
243 & aim: is for sports & 0.0998 & 0.1042 & 0.0044 \\
244 & color: is white & 0.0997 & 0.1071 & 0.0074 \\
245 & habitat: lives in trees & 0.0996 & 0.1089 & 0.0093 \\
246 & marking: has patches & 0.0985 & 0.1026 & 0.0041 \\
247 & size: is middle (compared to a mobile phone) & 0.0980 & 0.1028 & 0.0048 \\
248 & parts: has a jet engine & 0.0962 & 0.1017 & 0.0055 \\
249 & color: is black & 0.0961 & 0.0998 & 0.0037 \\
250 & function: can communicate in a long distance & 0.0960 & 0.1023 & 0.0063 \\
251 & outside color: is brown & 0.0956 & 0.0985 & 0.0029 \\
252 & edibility: is rare & 0.0952 & 0.0989 & 0.0037 \\
253 & usage scenarios: can be used in the lake & 0.0927 & 0.0980 & 0.0053 \\
254 & parts: has a large screen ($>10$ inches) & 0.0903 & 0.0959 & 0.0056 \\
255 & habitat: lives in the ocean & 0.0903 & 0.0943 & 0.0040 \\
256 & habitat: lives in the traveling & 0.0889 & 0.0944 & 0.0055 \\
257 & color: is red & 0.0872 & 0.0914 & 0.0042 \\
258 & function: can photograph & 0.0864 & 0.0919 & 0.0055 \\
259 & feeling: is naive & 0.0855 & 0.0886 & 0.0031 \\
260 & parts: has two wheels & 0.0840 & 0.0877 & 0.0037 \\
261 & tail: has a colorful tail & 0.0839 & 0.0887 & 0.0048 \\
262 & parts: has many usb interfaces & 0.0838 & 0.0888 & 0.0050 \\
263 & safety: is dangerous & 0.0836 & 0.0864 & 0.0028 \\
264 & parts: has a camera & 0.0831 & 0.0885 & 0.0053 \\
265 & marking: has stripes & 0.0827 & 0.0861 & 0.0035 \\
266 & taste: tastes puckery & 0.0824 & 0.0851 & 0.0027 \\
267 & aim: is for cleaning & 0.0808 & 0.0844 & 0.0036 \\
268 & epidermis: is angular & 0.0775 & 0.0800 & 0.0025 \\
269 & outside color: is transparent & 0.0769 & 0.0804 & 0.0035 \\
270 & epidermis: is barbed & 0.0737 & 0.0764 & 0.0027 \\
271 & shape: is schistose & 0.0731 & 0.0758 & 0.0027 \\
272 & neck: has a short neck & 0.0731 & 0.0780 & 0.0049 \\
273 & outside color: is black & 0.0730 & 0.0759 & 0.0029 \\
274 & marking: has spots & 0.0727 & 0.0756 & 0.0029 \\
275 & color: is blue & 0.0685 & 0.0718 & 0.0033 \\
276 & other body parts: has fins & 0.0644 & 0.0681 & 0.0037 \\
277 & material: is made of cloth & 0.0640 & 0.0664 & 0.0024 \\
278 & parts: has a track & 0.0630 & 0.0654 & 0.0024 \\
279 & usage scenarios: can be used on the track & 0.0628 & 0.0654 & 0.0027 \\
280 & current state: is shelled & 0.0624 & 0.0645 & 0.0021 \\
281 & smell: is smelly & 0.0621 & 0.0649 & 0.0028 \\
282 & paws: has hooves & 0.0618 & 0.0657 & 0.0039 \\
283 & appearance: has shells & 0.0618 & 0.0648 & 0.0030 \\
284 & function: can be located & 0.0613 & 0.0654 & 0.0041 \\
285 & color: is yellow & 0.0607 & 0.0639 & 0.0032 \\
286 & color: is transparent & 0.0602 & 0.0630 & 0.0028 \\
287 & color: is green & 0.0598 & 0.0626 & 0.0028 \\
288 & habit: is venomous & 0.0597 & 0.0617 & 0.0020 \\
289 & parts: has many network interfaces & 0.0577 & 0.0603 & 0.0026 \\
290 & limb: has tentacle & 0.0570 & 0.0594 & 0.0024 \\
291 & shape: is cubic & 0.0566 & 0.0591 & 0.0024 \\
292 & parts: has a mast & 0.0558 & 0.0586 & 0.0028 \\
293 & parts: has a wrist band & 0.0550 & 0.0567 & 0.0017 \\
294 & beak: has a curved beak & 0.0531 & 0.0569 & 0.0038 \\
295 & function: can refrigerate & 0.0530 & 0.0551 & 0.0022 \\
296 & habit: is nocturnal & 0.0525 & 0.0558 & 0.0034 \\
297 & shape: is cylindrical & 0.0511 & 0.0526 & 0.0015 \\
298 & other body parts: has a shell & 0.0506 & 0.0532 & 0.0026 \\
299 & aim: is for cleaning & 0.0506 & 0.0528 & 0.0022 \\
300 & shape: is ellipsoidal & 0.0501 & 0.0518 & 0.0017 \\
301 & power: consumes wind power & 0.0486 & 0.0502 & 0.0016 \\
302 & parts: has an antenna & 0.0483 & 0.0503 & 0.0019 \\
303 & beak: has a hooked beak & 0.0463 & 0.0498 & 0.0035 \\
304 & size: is middle (compared to pigs) & 0.0429 & 0.0455 & 0.0025 \\
305 & horn: has short horns & 0.0416 & 0.0439 & 0.0022 \\
306 & horn: has cuspidal horns & 0.0415 & 0.0437 & 0.0022 \\
307 & parts: has a sunroof & 0.0375 & 0.0389 & 0.0014 \\
308 & appearance: is mucous & 0.0364 & 0.0378 & 0.0014 \\
309 & color: is brown & 0.0358 & 0.0369 & 0.0010 \\
310 & color: is yellow & 0.0357 & 0.0369 & 0.0012 \\
311 & shape: is star-shaped & 0.0347 & 0.0358 & 0.0011 \\
312 & parts: has a touch pad & 0.0344 & 0.0370 & 0.0026 \\
313 & material: is made of wood & 0.0316 & 0.0329 & 0.0013 \\
314 & color: is red & 0.0307 & 0.0319 & 0.0012 \\
315 & horn: has long horns & 0.0303 & 0.0320 & 0.0017 \\
316 & function: can dive & 0.0300 & 0.0313 & 0.0013 \\
317 & color: is green & 0.0287 & 0.0305 & 0.0019 \\
318 & color: is green & 0.0282 & 0.0291 & 0.0009 \\
319 & outside color: is blue & 0.0276 & 0.0284 & 0.0008 \\
320 & outside color: is gray & 0.0275 & 0.0284 & 0.0009 \\
321 & appearance: is acicular & 0.0266 & 0.0280 & 0.0014 \\
322 & power: consumes manpower & 0.0262 & 0.0271 & 0.0008 \\
323 & price: is cheap & 0.0260 & 0.0269 & 0.0009 \\
324 & parts: has a sail & 0.0255 & 0.0267 & 0.0012 \\
325 & teeth: has tusks & 0.0169 & 0.0178 & 0.0009 \\
326 & wing: has short wings & 0.0166 & 0.0172 & 0.0007 \\
327 & horn: has curved horns & 0.0142 & 0.0150 & 0.0008 \\
328 & usage scenarios: can be used in the space & 0.0127 & 0.0131 & 0.0004 \\
329 & habitat: lives in the cave & 0.0107 & 0.0111 & 0.0004 \\
330 & habit: is the scavenger & 0.0107 & 0.0111 & 0.0004 \\
331 & beak: has a long beak & 0.0104 & 0.0109 & 0.0005 \\
332 & appearance: has scales & 0.0064 & 0.0066 & 0.0002 \\
333 & color: is blue & 0.0026 & 0.0028 & 0.0002 \\
334 & color: is red & 0.0020 & 0.0021 & 0.0001 \\
335 & color: is orange & 0.0015 & 0.0016 & 0.0001 \\
336 & current state: is peeled & 0.0012 & 0.0013 & 0.0001 \\
337 & color: is orange & 0.0012 & 0.0012 & 0.0001 \\
338 & color: is blue & 0.0012 & 0.0012 & 0.0001 \\
339 & color: is brown & 0.0011 & 0.0011 & 0.0001 \\
340 & color: is cyan & 0.0007 & 0.0008 & 0.0000 \\
341 & color: is violet & 0.0006 & 0.0007 & 0.0000 \\
342 & shape: is cubic & 0.0006 & 0.0007 & 0.0000 \\
343 & habitat: lives in fields & 0.0006 & 0.0007 & 0.0000 \\
344 & color: is violet & 0.0005 & 0.0006 & 0.0000 \\
345 & usage mode: is head-mounted & 0.0005 & 0.0006 & 0.0000 \\
346 & color: is cyan & 0.0005 & 0.0005 & 0.0000 \\
347 & color: is orange & 0.0005 & 0.0005 & 0.0000 \\
348 & color: is violet & 0.0005 & 0.0005 & 0.0000 \\
349 & color: is cyan & 0.0004 & 0.0004 & 0.0000 \\
350 & habitat: lives in the bush & 0.0003 & 0.0003 & 0.0000 \\
351 & existence: is rare & 0.0003 & 0.0003 & 0.0000 \\
352 & habitat: lives in the arctic & 0.0002 & 0.0003 & 0.0000 \\
353 & horn: has only one horn & 0.0002 & 0.0002 & 0.0000 \\
354 & shape: is globular & 0.0001 & 0.0001 & 0.0000 \\
355 & beak: has a dagger beak & 0.0000 & 0.0000 & 0.0000 \\
356 & smell: is smelly & 0.0000 & 0.0000 & 0.0000 \\
357 & habitat: lives in coastal places & 0.0000 & 0.0000 & 0.0000 \\
358 & habitat: lives in the desert & 0.0000 & 0.0000 & 0.0000 \\
359 & taste: tastes spicy & 0.0000 & 0.0000 & 0.0000 \\
\end{longtable}

\end{document}